\providecommand\FormatOverride{arxiv}
\providecommand\LogosOverride{camlsys-cambridge}
\providecommand\TitleBgOverride{cambridge}
\IfFileExists{preamble/format.tex}{}{}
\IfFileExists{iclr2027_conference.sty}{}{}
\IfFileExists{main.bbl}{}{}
\documentclass{paper}

\input{preamble/packages}
\input{preamble/macros}
\makeatletter

\def\@paper@title{Untitled}
\let\@paper@titlerunning\@empty
\let\@paper@keywords\@empty
\let\@paper@code\@empty
\let\@paper@plainauthors\@empty
\let\@paper@displayauthors\@empty
\let\@paper@displayaffils\@empty
\let\@paper@corrdisplay\@empty
\let\@paper@titlelogo\@empty
\let\@paper@titlemark\@empty
\def\paper@titlelogoheight{3.0em}
\def\paper@abstractfile{sections/abstract}
\newcounter{paper@affil}
\let\@neurips@authors\@empty
\let\@neurips@affils\@empty
\let\@neurips@corr\@empty
\let\@mlsys@authors\@empty
\let\@mlsys@affils\@empty
\let\@mlsys@corr\@empty
\newif\if@paper@firstauthor \@paper@firstauthortrue
\newif\if@paper@firstdisplayauthor \@paper@firstdisplayauthortrue
\newif\if@paper@hasequal \@paper@hasequalfalse
\newlength{\paper@titlecardinnerwidth}

\def\paper@titlecardpadding{0.38cm}
\def\paper@titlecardradius{10pt}

\definecolor{paperBrandCamlsysBlue}{HTML}{102344}
\definecolor{paperBrandCambridgeRed}{HTML}{EF2E31}
\definecolor{paperBrandCambridgeGold}{HTML}{FDDD04}
\definecolor{ThemeGoldBrown}{HTML}{A98426}
\colorlet{paperTitleBgCamlsys}{paperBrandCamlsysBlue!8!white}
\colorlet{paperTitleBgCambridgeGold}{paperBrandCambridgeGold!8!white}
\colorlet{paperTitleBgCambridge}{paperBrandCambridgeRed!3!paperTitleBgCambridgeGold}
\colorlet{paperTitleBgNeutral}{black!2!white}
\colorlet{paperTitleBgRQGM}{ThemeGoldBrown!5!white}
\colorlet{paperTitleFrameRQGM}{ThemeGoldBrown!35!white}
\colorlet{paperTitleBg}{paperTitleBgCamlsys}
\colorlet{paperTitleFrame}{paperTitleBg}
\definecolor{paperTitleFg}{HTML}{1F2328}
\definecolor{paperTitleMuted}{HTML}{4D5560}
\def\paper@titlebg@camlsys{camlsys}
\def\paper@titlebg@cambridge{cambridge}
\def\paper@titlebg@neutral{neutral}
\def\paper@titlebg@rqgm{rqgm-muted}

\csdef{paper@sym@equal}{*}

\newcommand{\paperTitle}[1]{\def\@paper@title{#1}}
\newcommand{\paperKeywords}[1]{\def\@paper@keywords{#1}}
\newcommand{\paperCode}[1]{\def\@paper@code{#1}}
\newcommand{\paperAbstractFile}[1]{\def\paper@abstractfile{#1}}
\newcommand{\paperTitleLogo}[1]{\def\@paper@titlelogo{#1}}
\newcommand{\paperTitleMark}[1]{\def\@paper@titlemark{#1}}
\newcommand{\paperTitleLogoHeight}[1]{\def\paper@titlelogoheight{#1}}
\newcommand{\paperTitleBackground}[1]{\edef\paper@titlebg@choice{#1}\ifx\paper@titlebg@choice\paper@titlebg@camlsys
    \colorlet{paperTitleBg}{paperTitleBgCamlsys}\colorlet{paperTitleFrame}{paperTitleBgCamlsys}\else
    \ifx\paper@titlebg@choice\paper@titlebg@cambridge
      \colorlet{paperTitleBg}{paperTitleBgCambridge}\colorlet{paperTitleFrame}{paperTitleBgCambridge}\else
      \ifx\paper@titlebg@choice\paper@titlebg@neutral
        \colorlet{paperTitleBg}{paperTitleBgNeutral}\colorlet{paperTitleFrame}{paperTitleBgNeutral}\else
        \ifx\paper@titlebg@choice\paper@titlebg@rqgm
          \colorlet{paperTitleBg}{paperTitleBgRQGM}\colorlet{paperTitleFrame}{paperTitleFrameRQGM}\else
          \ClassWarning{paper}{Unknown title-card background theme `\paper@titlebg@choice'; keeping current theme}\fi
      \fi
    \fi
  \fi
}
\paperTitleBackground{\paper@titlebg}
\newcommand{\paperTitleRunning}[1]{\def\@paper@titlerunning{#1}\ifpaper@mlsys \mlsystitlerunning{#1}\fi  }

\newcommand{\addAffiliation}[2]{\stepcounter{paper@affil}\csedef{paper@affil@#1}{\thepaper@affil}\ifnum\value{paper@affil}>\@ne \appto\@paper@displayaffils{\quad}\fi
  \protected@eappto\@paper@displayaffils{\noexpand\textsuperscript{\arabic{paper@affil}}}\appto\@paper@displayaffils{#2}\ifpaper@neurips
    \ifnum\value{paper@affil}>\@ne \appto\@neurips@affils{\quad}\fi
    \protected@eappto\@neurips@affils{\noexpand\textsuperscript{\arabic{paper@affil}}}\appto\@neurips@affils{#2}\fi
  \ifpaper@mlsys
    \appto\@mlsys@affils{\mlsysaffiliation{#1}{#2}}\fi
}

\newcommand{\paper@resolvekey}[1]{\ifcsdef{paper@sym@#1}{\ifstrequal{#1}{equal}{\@paper@hasequaltrue}{}\protected@eappto\@paper@syms{\csuse{paper@sym@#1}}}{\ifcsdef{paper@affil@#1}{\ifx\@paper@nums\@empty\else\appto\@paper@nums{,}\fi
      \protected@eappto\@paper@nums{\csuse{paper@affil@#1}}}{\ClassWarning{paper}{\noexpand\addAuthor: unknown key `#1'}}}}
\newcommand{\addAuthor}[2]{\ifx\@paper@plainauthors\@empty
    \appto\@paper@plainauthors{#1}\else
    \appto\@paper@plainauthors{, #1}\fi
  \let\@paper@nums\@empty \let\@paper@syms\@empty
  \forcsvlist{\paper@resolvekey}{#2}\if@paper@firstdisplayauthor\else\appto\@paper@displayauthors{,\space}\fi
  \appto\@paper@displayauthors{\mbox{#1}}\protected@eappto\@paper@displayauthors{\noexpand\textsuperscript{\@paper@nums\@paper@syms}}\@paper@firstdisplayauthorfalse
  \ifpaper@neurips
    \if@paper@firstauthor\else\appto\@neurips@authors{,\space}\fi
    \appto\@neurips@authors{#1}\protected@eappto\@neurips@authors{\noexpand\textsuperscript{\@paper@nums\@paper@syms}}\@paper@firstauthorfalse
  \fi
  \ifpaper@mlsys
    \appto\@mlsys@authors{\mlsysauthor{#1}{#2}}\fi
}

\newcommand{\correspondingAuthor}[2]{\appto\@paper@corrdisplay{#1 \texttt{<#2>}}\ifpaper@neurips \appto\@neurips@corr{\thanks{Correspondence to: #1 \texttt{<#2>}}}\fi
  \ifpaper@mlsys   \appto\@mlsys@corr{\mlsyscorrespondingauthor{#1}{#2}}\fi
}

\newcommand{\paper@setneuripspdfauthor}{\hypersetup{pdfauthor={\@paper@plainauthors}}}
\newcommand{\paper@setmlsyspdfauthor}{\hypersetup{pdfauthor={\@paper@plainauthors}}}
\newcommand{\paper@setpdfmetadata}{\@ifpackageloaded{hyperref}{\hypersetup{pdftitle={\@paper@title}}\ifx\@paper@keywords\@empty\else
      \hypersetup{pdfkeywords={\@paper@keywords}}\fi
    \ifpaper@neurips
      \paper@setneuripspdfauthor
    \fi
    \ifpaper@mlsys
      \paper@setmlsyspdfauthor
    \fi
    \ifpaperanon
      \hypersetup{pdfauthor={},pdfsubject={},pdfkeywords={}}\else\ifpapericlr
      \hypersetup{pdfauthor={\@paper@plainauthors}}\fi\fi
  }{}}

\newcommand{\paper@titleLogoGraphic}{\raisebox{0.27pt}{\includegraphics[height=4.28em,width=\linewidth,keepaspectratio]{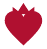}}}

\newcommand{\paper@titleLine}{\ifx\@paper@titlelogo\@empty
    \ifx\@paper@titlemark\@empty
      {\raggedright\sffamily\bfseries\LARGE\color{paperTitleFg}\@paper@title\par}\else
      \paper@titleLineWithMark
    \fi
  \else
    \paper@titleLineWithMark
  \fi
}

\newcommand{\paper@titleLineWithMark}{\noindent
  \begin{minipage}[c]{0.12\linewidth}
    \centering\paper@titleLogoGraphic
  \end{minipage}\hspace{0.02\linewidth}\begin{minipage}[c]{0.85\linewidth}
    {\raggedright\sffamily\bfseries\fontsize{15.5}{18.5}\selectfont\color{paperTitleFg}\@paper@title\par}\end{minipage}\par
}

\newcommand{\paper@identityBlock}{\ifx\@paper@displayauthors\@empty\else
    {\raggedright\sffamily\bfseries\fontsize{8.2}{10}\selectfont\color{paperTitleFg}\@paper@displayauthors\par}\fi
  \ifx\@paper@displayaffils\@empty\else
    \vspace{0.06cm}{\raggedright\footnotesize\color{paperTitleMuted}\@paper@displayaffils
      \if@paper@hasequal \quad \textsuperscript{*}Equal contribution.\fi\par}\fi
}

\newcommand{\paper@metadataText}{\ifx\@paper@keywords\@empty\else
    {\bfseries Keywords:} \@paper@keywords
  \fi
  \ifx\@paper@code\@empty\else
    \ifx\@paper@keywords\@empty\else\quad\fi
    {\bfseries Code:} \@paper@code
  \fi
}

\newcommand{\paper@metadataLogoRow}{\noindent
  \begin{minipage}[c]{0.6\linewidth}
    {\scriptsize\raggedright\color{paperTitleMuted}\paper@metadataText}\end{minipage}\hfill
  \begin{minipage}[c]{0.4\linewidth}
    \raggedleft\paperLogoStrip
  \end{minipage}\par
}

\newcommand{\paper@titlecardLogoOverlay}{}

\newcommand{\paper@titlecard}{\begingroup
    \setlength{\paper@titlecardinnerwidth}{\dimexpr\linewidth-\paper@titlecardpadding-\paper@titlecardpadding\relax}\noindent
    \begin{tikzpicture}
      \node[
        fill=paperTitleBg,
        draw=paperTitleFrame,
        line width=0.6pt,
        rounded corners=\paper@titlecardradius,
        inner sep=\paper@titlecardpadding,
        outer sep=0pt,
        text width=\paper@titlecardinnerwidth,
        align=justify
      ] (paperTitleCardNode) {\begin{minipage}{\paper@titlecardinnerwidth}
          \setlength{\parindent}{0pt}\setlength{\parskip}{0pt}\paper@titleLine
          \vspace{0.12cm}\paper@identityBlock
          \vspace{0.14cm}{\normalsize\color{paperTitleFg}{\sffamily\bfseries Abstract.}\space Self-improving agents are state-of-the-art on agentic coding benchmarks, yet their search methods assume a stationary evaluation criterion. This ignores a central feature of evolution: species adapt as their environments change with them. We introduce the \textbf{Red Queen G\"odel Machine (\ourmethod)}, an evolutionary framework for recursive self-improvement under non-stationary utilities. This allows learned evaluators to improve alongside the agents they guide. On \deepswe, \theourmethod improves over its fixed-evaluator baseline by adding a complementary agent-as-a-judge code-review signal: a co-evolved reviewer grades coder patches to guide search. At low reasoning effort, the \ourmethod coder passes \textbf{82.1}$\boldsymbol{\%}$ of held-out tasks against the baseline's \textbf{75.0}$\boldsymbol{\%}$, and nearly matches the \astra model at high effort. In scientific paper writing and reviewing, and Olympiad-level proof writing and grading, co-evolved evaluators provide an evaluation criterion. Anchored to human IMO grades, a co-evolved grader writes its own milestone rubric and exceeds static baselines at a \textbf{3}$\boldsymbol{\times}$ lower search cost, driving the prover to the best mean score. Since \theourmethod can modify the search objective across epochs, it can regularize the search. For example, \theourmethod reduces self-preference bias via an additional adversarial objective to discover reviewers equally stringent on \AI{} and human work. Guided by these calibrated reviewers, co-evolved writers reach \textbf{1.78}$\boldsymbol{\times}$--\textbf{1.86}$\boldsymbol{\times}$ higher acceptance rates than the baseline under an agent-as-a-judge panel. \Theourmethod enables self-improving systems where agents and evaluators recursively bootstrap each other beyond static evaluation.\par}\vspace{0.14cm}\paper@metadataLogoRow
        \end{minipage}};\paper@titlecardLogoOverlay
    \end{tikzpicture}\par\vspace{0.12cm}\endgroup
}

\newcommand{\paper@setmlsysrunninghead}{\ifx\@paper@titlerunning\@empty
    \gdef\@mlsystitlerunning{\@paper@title}\else
    \gdef\@mlsystitlerunning{\@paper@titlerunning}\fi
  \chead{\small\bf\@mlsystitlerunning}\thispagestyle{empty}}

\newcommand{\paper@mlsysNotice}{{\let\thefootnote\relax\footnotetext{\hspace*{-\footnotesep}\Notice@String}}}

\newcommand{\paper@iclrTitle}{\ifx\@paper@titlelogo\@empty\else
    \raisebox{-0.12\height}{\includegraphics[height=0.86\baselineskip,keepaspectratio]{\@paper@titlelogo}}\hspace{0.35em}\fi
  \@paper@title
}

\newcommand{\makepapertitle}{\paper@setpdfmetadata
  \ifpapericlr
    \typeout{ICLR textwidth=\the\textwidth}\typeout{ICLR textheight=\the\textheight}\typeout{ICLR left=\the\dimexpr1in+\hoffset+\oddsidemargin\relax}\typeout{ICLR top=\the\dimexpr1in+\voffset+\topmargin+\headheight+\headsep\relax}\typeout{ICLR fontsize=\f@size}\typeout{ICLR baselineskip=\the\baselineskip}\typeout{ICLR fontfamily=\f@family}\ifpaperanon
      \title{\paper@iclrTitle}\author{Anonymous authors}\maketitle
      \begin{abstract}Self-improving agents are state-of-the-art on agentic coding benchmarks, yet their search methods assume a stationary evaluation criterion. This ignores a central feature of evolution: species adapt as their environments change with them. We introduce the \textbf{Red Queen G\"odel Machine (\ourmethod)}, an evolutionary framework for recursive self-improvement under non-stationary utilities. This allows learned evaluators to improve alongside the agents they guide. On \deepswe, \theourmethod improves over its fixed-evaluator baseline by adding a complementary agent-as-a-judge code-review signal: a co-evolved reviewer grades coder patches to guide search. At low reasoning effort, the \ourmethod coder passes \textbf{82.1}$\boldsymbol{\%}$ of held-out tasks against the baseline's \textbf{75.0}$\boldsymbol{\%}$, and nearly matches the \astra model at high effort. In scientific paper writing and reviewing, and Olympiad-level proof writing and grading, co-evolved evaluators provide an evaluation criterion. Anchored to human IMO grades, a co-evolved grader writes its own milestone rubric and exceeds static baselines at a \textbf{3}$\boldsymbol{\times}$ lower search cost, driving the prover to the best mean score. Since \theourmethod can modify the search objective across epochs, it can regularize the search. For example, \theourmethod reduces self-preference bias via an additional adversarial objective to discover reviewers equally stringent on \AI{} and human work. Guided by these calibrated reviewers, co-evolved writers reach \textbf{1.78}$\boldsymbol{\times}$--\textbf{1.86}$\boldsymbol{\times}$ higher acceptance rates than the baseline under an agent-as-a-judge panel. \Theourmethod enables self-improving systems where agents and evaluators recursively bootstrap each other beyond static evaluation.\end{abstract}\else
      \renewcommand{\headrulewidth}{0pt}\par
      \paper@titlecard
    \fi
  \fi
  \ifpaper@neurips
    \par
    \begingroup
      \thispagestyle{empty}\paper@titlecard
      \@notice
    \endgroup
  \fi
  \ifpaper@mlsys
    \paper@setmlsysrunninghead
    \twocolumn[\paper@titlecard
      \vskip 0.18in
    ]\paper@mlsysNotice
  \fi
}

\makeatother

\ifpaperanon\else\makeatletter
\newcommand{\paperLogoStrip}{\begingroup
    \raisebox{-0.5\height}{\includegraphics[height=6.2mm]{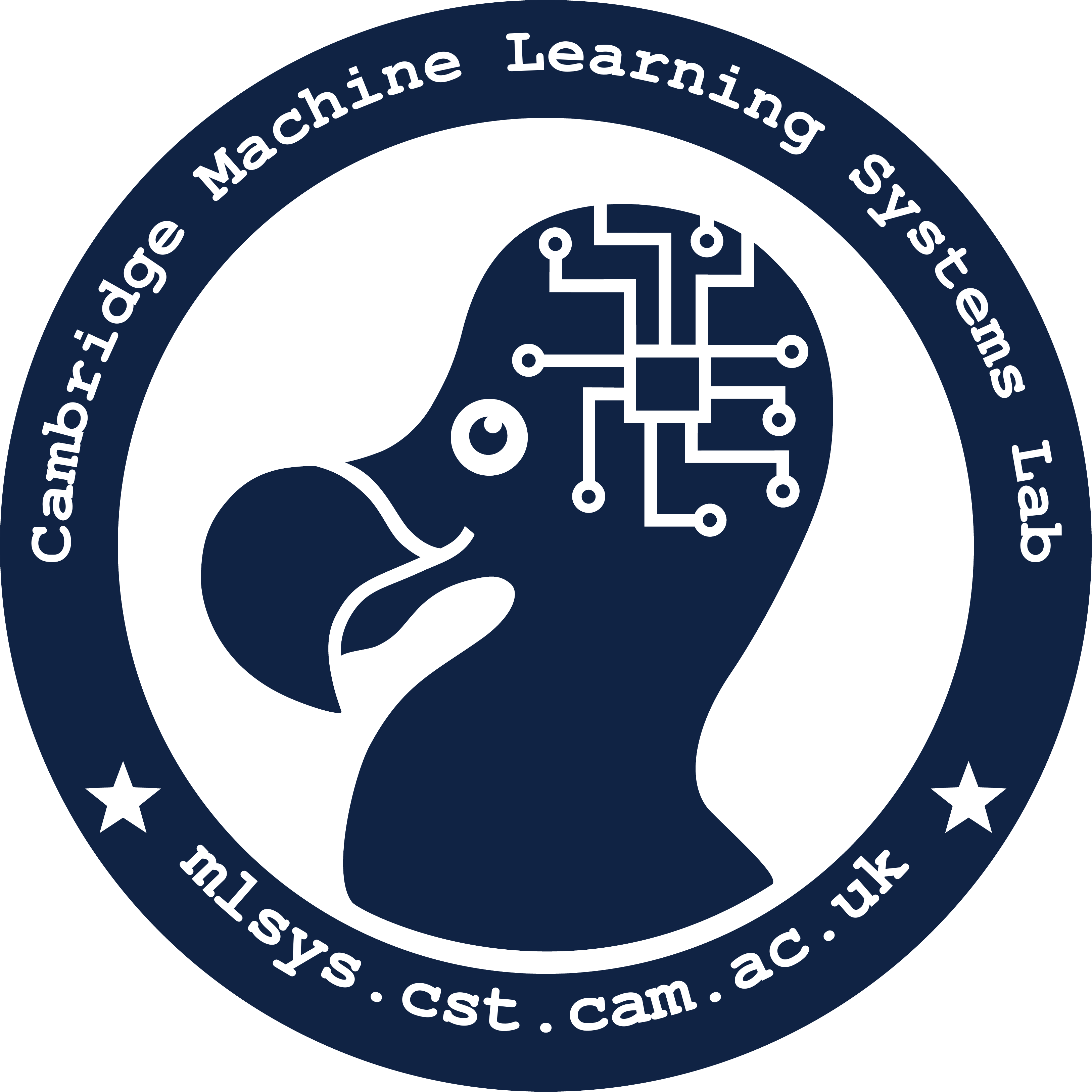}}\hspace{0.55em}\raisebox{-0.5\height}{\includegraphics[height=5.5mm]{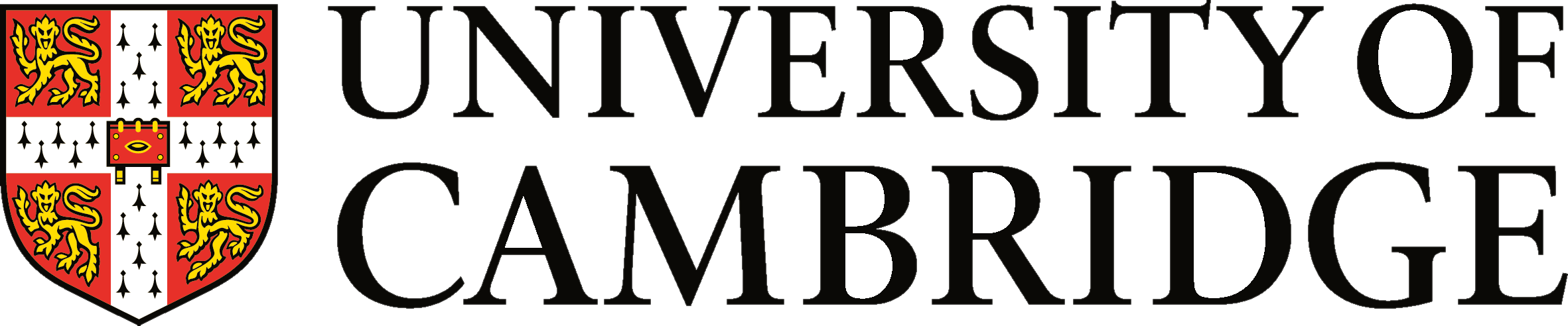}}\hspace{0.55em}\raisebox{-0.5\height}{\includegraphics[height=6.2mm]{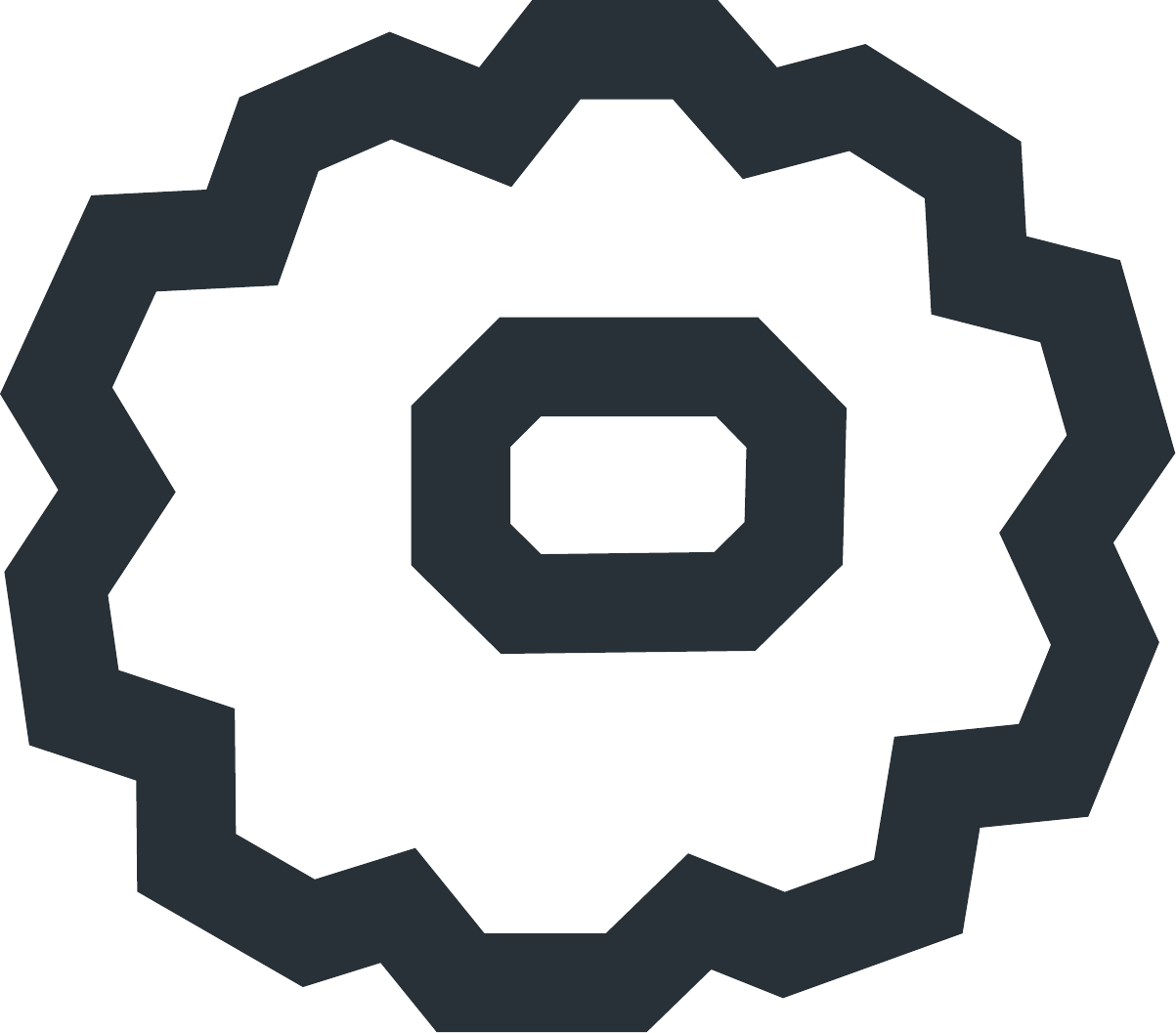}}\endgroup
}

\makeatother
\fi

\paperTitle{The Red Queen G\"odel Machine: Co-Evolving Agents and Their Evaluators}
\paperTitleRunning{The Red Queen G\"odel Machine}
\paperKeywords{self-improving agents, learned evaluation, multi-agent systems, automated scientific discovery, controlled utility evolution, co-evolutionary search}
\paperAbstractFile{sections/abstract}
\paperTitleBackground{rqgm-muted}
\paperTitleLogo{branding/red-queen-mark.pdf}
\paperTitleLogoHeight{4.28em}

\addAffiliation{cambridge}{University of Cambridge}
\addAffiliation{nvidia}{NVIDIA}
\addAffiliation{flower}{Flower Labs}
\addAffiliation{mbzuai}{MBZUAI}
\addAffiliation{inria}{Inria}
\addAuthor{Alex Iacob}{cambridge,equal}
\addAuthor{Andrej Jovanovi\'c}{cambridge,flower,equal}
\addAuthor{William F. Shen}{cambridge,equal}
\addAuthor{Daniel Burkhardt}{nvidia}
\addAuthor{Meghdad Kurmanji}{cambridge}
\addAuthor{Nurbek Tastan}{mbzuai}
\addAuthor{Lorenzo Sani}{cambridge,flower}
\addAuthor{Niccolò Alberto Elia Venanzi}{nvidia}
\addAuthor{Jiayi Nie}{cambridge}
\addAuthor{Ambroise Odonnat}{inria}
\addAuthor{Zeyu Cao}{cambridge}
\addAuthor{Bill Marino}{cambridge}
\addAuthor{Xinchi Qiu}{cambridge}
\addAuthor{Rika Antonova}{cambridge}
\addAuthor{Nicholas D. Lane}{cambridge,flower}
\correspondingAuthor{Alex Iacob}{aai30@cam.ac.uk}

\begin{document}
\makepapertitle
\etocdepthtag.toc{main}
\ifpapericlr\setcounter{totalnumber}{2}\fi

\FloatBarrier

\begin{figure}[H]
  \centering
\includegraphics[width=\linewidth]{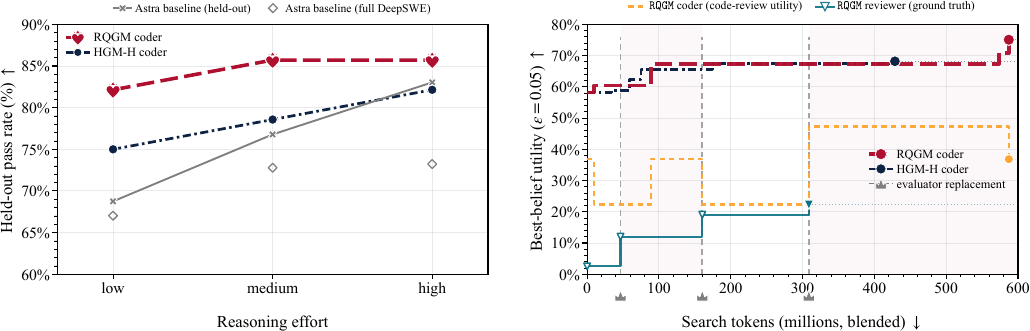}
\caption{\textbf{\ourmethod\ outperforms the \hgmh baseline on \deepswe: it matches \hgmh at equal search cost and keeps improving after \hgmh plateaus, with harness improvements that transfer to higher reasoning efforts.} \emph{Left:} held-out pass rate of the best coder harness across reasoning efforts; both \theourmethod and \hgmh searched using low reasoning effort for the coder. Gray $\times$'s show base \astra held-out performance; diamonds show its performance on full \deepswe. \emph{Right:} the leading agent's best-belief task utility during search, defined as the 5th percentile of its Beta working posterior. Orange shows the leading coder's utility under the co-evolved code-review objective; teal tracks the leading reviewer's utility against ground truth. \Theourmethod co-evolves a learned code reviewer, providing a complementary search signal that improves across evolutionary epochs. Crowns mark epoch boundaries, where a stronger reviewer replaces the incumbent and reviewer-dependent utility records are erased, so coder-review utility may shift even when the leading coder is unchanged; these boundaries coincide with improvements in the teal curve.}
  \label{fig:intro-token-efficiency}
\end{figure}
\section{Introduction}
\label{sec:intro}
Self-improving agents aim to turn gains in coding and reasoning into a recursive loop, where each stronger agent can produce better variants of itself. Search methods such as the Darwin and Huxley G\"odel Machines (\dgm and \hgm)~\citep{zhang2025darwin,wang2025huxley} set the open-source state-of-the-art on agentic coding tasks by editing their own code and retaining variants that improve an external utility signal. \hyperagents~\citep{zhang2026hyperagents} then extended this self-improvement beyond coding to task agents in arbitrary domains. Despite this progress, current search methods still depend on stationary evaluation fixed outside the improvement loop. This is unlike biological evolution, where, under Van Valen's Red Queen hypothesis~\citep{van1973new}, each species adapts to competitors that evolve in turn. This dependence constrains self-improvement when a target task has no direct benchmark, as for paper and proof writing, unlike paper reviewing and proof grading~\citep{lu2024ai,yamada2025ai,luong2025towards}, when evaluation is slow or weakly informative~\citep{levine2018learning,burger2020mobile,NIPS2017_453fadbd}, and when static benchmarks saturate or become vulnerable to reward hacking~\citep{yang2026selfimprovement,DBLP:journals/corr/abs-2203-04592}.

We introduce the \textbf{Red Queen G\"odel Machine (\ourmethod)}, an evolutionary framework that treats evaluation as part of the search process. We study this through \emph{co-evolved} learned evaluators, which improve alongside the task agents they guide. Such an evaluator can supply the search signal when no benchmark exists, add criteria a benchmark cannot capture, or act as a cheaper proxy. The core mechanism is \emph{controlled utility evolution}, which divides search into \emph{evolutionary epochs}. Within an epoch, one evaluator is frozen and grades every task agent, supplying a stationary utility signal. In parallel, the search co-evolves challenger evaluators with their task agents and scores them against a held-out ground-truth dataset. At an epoch boundary, a challenger replaces the frozen evaluator only if it raises the incumbent's best-belief score on a replacement criterion; \theourmethod then applies \emph{selective erasure}, discarding only utility records that depended on the replaced evaluator. Each epoch is thus a \emph{fixed-criterion search problem}, so \hgm's guarantees apply within each epoch~(\cref{sec:method,app:proofs}). Since validity requires only a fixed within-epoch utility and boundary erasure of dependent records, time-varying benchmarks and adversarial objectives fit the same structure.

On \deepswe~\citep{huang2026deepswe}, the \ourmethod coder running at low reasoning reaches a held-out pass rate of $\textbf{82.1}\boldsymbol{\%}$ against the \hgmh baseline at $75.0\boldsymbol{\%}$. \Theourmethod with low reasoning also matches the base \astra model in the mini-SWE-agent harness at high effort, and reaches $\textbf{85.7}\boldsymbol{\%}$ at medium or high effort~(\cref{fig:intro-token-efficiency}). On \polyglot, \theourmethod reaches $\textbf{71.7}\boldsymbol{\%}$ against $69.9\boldsymbol{\%}$ with fewer search tokens~(\cref{fig:polyglot-token-efficiency}). This illustrates a broader advantage of \theourmethod: shared expansions can improve multiple utilities at once, since most accepted expansions edit code shared by the coder and reviewer~(\cref{fig:analysis-transfer}). We then turn to scientific paper writing and reviewing, and Olympiad-level proof writing and grading, where only reviewing and grading have a benchmark. Co-evolved writers raise reviewer-panel acceptance from $21.8\boldsymbol{\%}$ to $\textbf{40.5}\boldsymbol{\%}$~(\cref{tab:paper-writer-cross-reviewer}), while a co-evolved grader exceeds the static baselines with a $\textbf{3}\boldsymbol{\times}$ lower search cost than \hgmh~(\cref{fig:prover-grader-results}). Co-evolution can also regularize evaluators: replaying \AI-generated papers accepted by a displaced reviewer as adversarial samples produces a reviewer equally stringent on machine-written and human-written work~(\cref{fig:paper-writer-reviewer-results}). Given the pace of model improvement, we believe automatically evolving the benchmarking process is needed for AI systems to exceed what static evaluation alone can achieve.

\introtakeawaybox{
\begin{enumerate}[leftmargin=*]
\item \textbf{Recursive Self-Improvement under Non-Stationary Utilities.}
\Theourmethod extends self-improving agents beyond fixed evaluation by allowing the utility to change over the course of search~(\cref{sec:method}). This supports both \emph{learned evaluators} that act as agents-as-judges for downstream tasks and more general time-varying utilities, such as curricula over task and dataset partitions during search.
\item \textbf{Controlled Utility Evolution.} Building on the above, this framework extends self-improvement to learned evaluators while anchoring progress to a ground-truth benchmark. However, benchmark-based supervision alone may impose a ceiling on evaluator quality. To move beyond this limit, we introduce additional selective pressure through adversarial objectives that regularize the search process, encouraging evaluators to continue improving beyond the level directly supported by the ground-truth signal.
\item \textbf{Practical Benefits of Co-Evolved Learned Evaluators and Search Interventions.}
Co-evolved learned evaluators improve both performance and search efficiency, yielding stronger coders than fixed-evaluator search, enabling search where direct benchmark feedback is unavailable, and supporting adversarial refinement of evaluators to correct failure modes such as over-leniency.
\end{enumerate}
}
\FloatBarrier
\label{paper:intro-end}

\section{Related Work}
\label{sec:related}

\paragraph{Self-Improving AI.}
\looseness=-1
The G\"odel Machine \citep{schmidhuber2003godel} is a theoretical construct that rewrites its own code whenever it can \emph{prove} that doing so is beneficial. A recent line of work makes this construct empirical by replacing proof search with archive search over observed utility. The Darwin G\"odel Machine \citep[\dgm,][]{zhang2025darwin} ranks self-modifying coding agents by benchmark utility, the Huxley-G\"odel Machine \citep[\hgm,][]{wang2025huxley} scores nodes by \emph{clade metaproductivity} rather than their own score, and HyperAgents \citep{zhang2026hyperagents} extends self-modification beyond coding tasks by making each node an editable meta-agent/task-agent codebase. These systems hold their evaluation criterion \emph{fixed} outside the loop, exposing them to benchmark saturation \citep{DBLP:journals/corr/abs-2203-04592}, judge biases \citep{zheng2023judging,dubois2024lengthcontrolled}, and reward hacking \citep{skalse2022defining}. \Theourmethod departs from this line of work by \textit{pulling the evaluation criterion itself inside the loop}, searching over the evaluators that define the utility.

\paragraph{Co-Evolving Evaluation.}
Recent work co-adapts generators and rewards within the gradient-based training of a single model \citep{yuan2024selfreward,wang2025cure,zhu2026selfevolvingdeepresearchjoint} or co-evolves evaluation metrics with agent skills under ground-truth guards \citep{zhang2026gradesgradercoevolvingevaluation}. By contrast, \theourmethod operates in the search space of agentic codebases, preserving \hgm's epoch-local validity while gating every evaluator promotion on an evaluator-independent ground-truth anchor. In doing so, it guards against the self-referential drift documented in self-rewarding training \citep{DBLP:conf/nips/PanicksseryBF24,pan2024feedback}. A further survey of related work can be found in App.~\ref{app:related}.

\section{\ourmethod: Co-Evolving Agents and Their Evaluators}
\label{sec:method}

As visualized in \cref{fig:method-overview}, \theourmethod is a search over an archive $\mathcal{T}$ where each node is a shared, evolvable workspace containing task agents, evaluators, and a meta-agent that oversees the workspace. Expansion creates an edited child starting from the parent. Since the meta-agent and role-agent code are editable, roles may reuse each other's code if the meta-agent deems it useful. Only the scoring and orchestration harness are fixed. Departing from prior work, \theourmethod introduces \emph{controlled utility evolution} to co-evolve agents and their evaluators~(\cref{sec:method:transitions}). Details are in \cref{app:algo}.

\begin{figure}[htb]
  \centering
  \includegraphics[width=\linewidth]{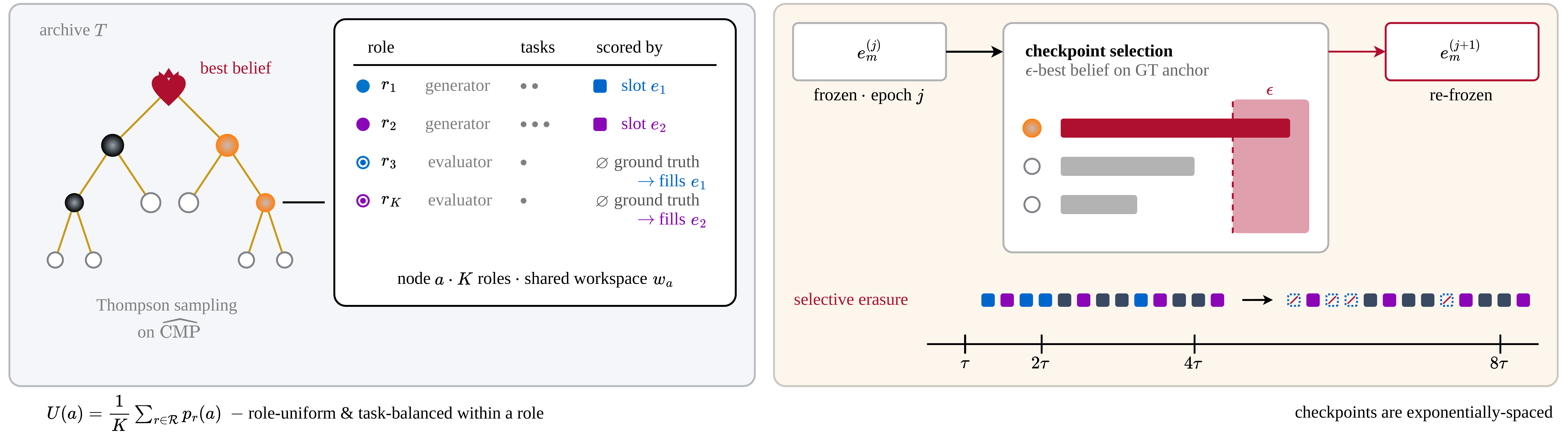}
  \caption{\textbf{\ourmethod co-evolves task agents and evaluators in a shared workspace tree.} Evaluators are scored on ground-truth. At checkpoints, the highest best-belief evaluator is selected, and records dependent on displaced evaluators are erased~(\cref{alg:main}), we show an expanded view in \cref{fig:method-expanded}.}
  \label{fig:method-overview}
\end{figure}

\subsection{Search Space and Utility}
\label{sec:method:formal}

\Theourmethod enables the evaluation of tasks without ground-truth signal via distinct evaluator roles. In addition to evaluator-independent roles that use fixed benchmarks~(\cref{sec:method:multiagent}), \emph{evaluator-dependent} roles use \emph{learned evaluators}, each held in an independently replaced \emph{evaluator slot}.

\noindent\textbf{Self-improvement as tree search.}~~Following \citet{wang2025huxley}, the archive $\mathcal{T}$ at time $t$ ($\mathcal{T}_t$) grows from a root agent $a_0$. Each step either \emph{expands} a node, adding an edited child, or \emph{evaluates} a node on one task, adding a binary outcome to its successes $S_a$ or failures $F_a$. Expansion occurs via a UCB-Air gate \citep{NIPS200849ae49a2} when the number of evaluations is sufficient (see \cref{sec:method:hierarchy}). A node's own success rate poorly predicts what its descendants achieve, so Thompson sampling selects expansion parents by \emph{clade metaproductivity}, the rate pooled over the subtree $C(a)$ rooted at $a$, $\widehat{\cmp}(a)=S^{C}_a/(S^{C}_a+F^{C}_a)$, and evaluation targets by their own $(S_a,F_a)$. Under a fixed criterion, a \cmp oracle suffices to implement the G\"odel Machine~(\cref{thm:cmp}).

\noindent\textbf{Utility definition.}~~The role-aggregated search utility averages success rates over each role's tasks, then over roles. Let $p_{r,d}(a)$ be node $a$'s expected success rate on task $d$ of role $r$ under the fixed evaluation criteria. With $K$ roles and task pools $\mathcal D_r$, we weight roles uniformly:
\[
 U(a)=\frac{1}{K}\sum_{r=1}^{K}
 \left[\frac{1}{|\mathcal D_r|}\sum_{d\in\mathcal D_r}p_{r,d}(a)\right].
\]
At each evaluation, the scheduler selects a node, then its least-evaluated eligible role and task (\cref{sec:method:hierarchy}). For any fixed utility, the node's $\epsilon$-best-belief is a lower quantile of a Beta working posterior, and the search returns its maximizer once the budget $B$ is spent.
\[
 BB_\epsilon(a)
 =\operatorname{Quantile}_{\epsilon}\!\left[\operatorname{Beta}(1+S_a,\,1+F_a)\right].
\]
If all roles and tasks remain eligible and each pair's empirical success rate converges to its fixed mean, the pooled rate converges to $U(a)$ under balanced sampling~(\cref{prop:balanced-working-posterior}).

\subsection{Controlled Utility Evolution}
\label{sec:method:transitions}

In fixed-criterion search, \emph{learned evaluators} cannot co-evolve with the agents they score. To enable the co-evolution of evaluators and their agents, we first assume \emph{epoch-local stationarity}: each evaluator, node workspace, and scoring rule stays fixed within an epoch, the interval between two checkpoints~(\cref{assumption:generative}). At specific checkpoints (epoch boundaries), learned evaluators may be replaced if they fulfill the following criterion~(\cref{sec:method:generative}).

\noindent\textbf{Ground-truth evaluator replacement.}~~At each checkpoint, every evaluator slot compares its incumbent against challengers on a \emph{ground-truth anchor}, a fixed held-out dataset of objective or human labels. The highest anchor best-belief wins, with ties retaining the incumbent. A slot may also reward rejecting an \emph{adversarial pool} of outputs its displaced evaluator had accepted~(\cref{sec:experimental-design}); without one, the evaluator-independent anchor keeps one criterion across epochs~(\cref{app:swap-rule}).

\theorybox[Anchor Selection.]{Let $Q(e)$ be evaluator $e$'s exact score on the fixed anchor, and $\mathcal E$ a finite candidate set containing the incumbent $e_{\mathrm{old}}$. Exact comparison on the full validation set raises $Q$ at every replacement~(\cref{prop:campaign-exact}). If each candidate's estimate is within $\eta$ of $Q(e)$ and $e_{\mathrm{new}}$ has the largest estimate, then
\[
 Q(e_{\mathrm{new}})\ge\max_{e\in\mathcal E}Q(e)-2\eta
 \ge Q(e_{\mathrm{old}})-2\eta.
\]
Each of the two score approximations contributes at most $\eta$ to the comparison error~(\cref{prop:campaign-approximation}).}
For fixed $\epsilon\in(0,1)$, the Beta mean and quantile converge to $Q(e)$ when their samples consistently estimate that same objective. Each utility record $z$ carries the set $\operatorname{dep}(z)$ of slots whose evaluator affected its generation, sampling, replay, or scoring. Upon replacement on slot $m$, \emph{selective erasure} removes every record $z$ with $m\in\operatorname{dep}(z)$ written under the displaced evaluator, since keeping or rescaling it would mix evidence from two utilities. Other records and all artifacts are unaffected, and node and clade statistics are recomputed~(\cref{def:utility-transition,prop:erasure-stationarity}).
\theorybox[Epoch-Local Validity.]{Stationary generation and scoring give each node--role--task tuple a fixed success probability. With selective erasure, each epoch is a fixed binary-outcome search problem~(\cref{prop:validity,prop:piecewise-fixed-validity}), to which \hgm's analysis applies within the epoch under its remaining assumptions~(\cref{assumption:hgm}).}

\noindent\textbf{Amortized utility-transition cost.}~~Re-scoring the entire archive after every erasure wastes budget on nodes the search may never revisit. To avoid this, \theourmethod re-scores old nodes on demand, using cached outputs without additional task-agent calls beyond the new evaluator. In addition, we use exponentially spaced checkpoints to bound total cost (\cref{sec:method:utility}).

\theorybox[Bounded Recovery.]{Fix one evaluator slot and a budget of $B$ evaluations. If each checkpoint reprocesses at most all prior slot-dependent records, the total is bounded by $\rho B/(\rho-1)$ for checkpoint ratio $\rho>1$. Reassessing all prior records after every evaluation would require $\mathcal O(B^2)$ work~(\cref{prop:linear-work,rem:dense-quadratic}).}

\section{Experimental Design}
\label{sec:experimental-design}

\noindent\textbf{Domains and ground-truth anchors.}~~Paper writing and review are paired, anchored to \apres accept/reject decisions~\citep{zhao2026apres}. Proof writing and grading are anchored to \imogradbench human grades~\citep{luong2025towards}, where a proof is accepted only when the epoch's frozen grader awards full credit ($7$ of $7$). Coding uses two benchmarks, each with two anchors. On \polyglot~\citep{gauthier2024polyglot}, the coding agent is anchored to executable tests, with the code reviewer anchored to \crave~\citep{zhang2025crave}, a dataset of accepted/rejected pull requests. On \deepswe~\citep{huang2026deepswe}, the coder is anchored to the benchmark's tests, and the code reviewer to a balanced dataset of publicly available accepted/rejected patches~(\cref{app:splits}). Reward design follows \hyperagents~\citep{zhang2026hyperagents}: a task agent writes a reward function, and a quadruped policy trained on it with PPO is scored on forward walking during search and on jumping held out. Its utility is continuous, so we slightly adapt the \hgm utility and search prior~(\cref{app:continuous}). \Cref{app:reward-design} details the splits, curriculum, and protocol.

\noindent\textbf{Data splits.}~~Every dataset is divided into train, validation, and held-out test splits. The meta-agent may request task-agent results only on the train split, which informs its edits. Node selection and evaluator replacement are driven solely by the validation split. The held-out test split is never seen during search and is used only for post-hoc scoring of search-selected agents. \Cref{app:splits} describes how each dataset is split. We show the initialization of each role in \cref{app:prompts}.

\noindent\textbf{Baselines.}~~Every headline domain uses a learned baseline, \hgmh~\citep{wang2025huxley,zhang2026hyperagents}, which replaces the \dgm search algorithm of \hyperagents~\citep{zhang2026hyperagents} with the more sample-efficient \hgm, while keeping evaluators frozen. We use the \sakana prompt reviewer~\citep{lu2024ai} for paper writing, \proofautograder~\citep{luong2025towards} for proof grading, and the \imocode prover~\citep{huang2025winninggoldimo2025} for proof writing. We evaluate the best published \hyperagents~(\dgmh) patches from \citet{zhang2026hyperagents} for paper reviewing and grading as a proxy for \dgm-level performance. We selected \gptlow (low) for experiments unless otherwise stated. For \deepswe, we move to \astra~\citep{openai2026gpt6astra}, the strongest model on the \deepswe leaderboard at the time of writing, with a high-effort meta-agent and low-effort task agents in both arms. For the reward design domain we use \nemotronultra.

\noindent\textbf{Search interventions.}~~In paper writing/reviewing and on \deepswe, we use controlled utility evolution to challenge evaluators with outputs accepted by previous evaluators. After replacement, the papers or patches the displaced evaluator had accepted form an adversarial pool. Subsequent evaluator selection rewards rejecting these generated outputs alongside performance on the labeled examples. We ablate these mechanisms in \cref{app:nemotron-ablations}. For reward design we use a curriculum over the difficult of the task as a form of evolving utility.

\noindent\textbf{Costs and reporting.}~~Best-belief is comparable across epochs only when it rests on a fixed ground-truth anchor, so a \emph{global} best-belief winner exists only for anchored roles: the learned evaluators and the verifiable coding agent. The paper writer and prover have no anchor and admit only \emph{epoch-local} winners, each the best agent under its epoch's frozen evaluator; we select these per epoch and score them post-hoc against fixed external judges~(\cref{tab:paper-writer-cross-reviewer,tab:imo-proof-cross-grader}). We compare headline runs at matched search tokens rather than evaluation counts, since evaluations differ in cost; we report \emph{blended tokens}: input plus output, with output weighted at $5\times$ input cost~\citep{openai2026pricing,anthropic2026pricing}, counting both generation and evaluation calls. Following \citet{wang2025huxley}, a best-belief agent's cost is measured when it first attains its highest $\epsilon$-best-belief score, not when the epoch or run finishes. \emph{Specialist} nodes are those that maximize best-belief on one target utility; a \emph{generalist} is a node that maximizes the average best-belief across its active utilities~(\cref{app:baselines-selection}). When the specialist and generalist coincide on the same node, we report a single result.

\section{Results}
\label{sec:experiments}

Our results show that where ground truth exists, a co-evolved evaluator provides a complementary objective and finds stronger agents than fixed-evaluator search~(\cref{sec:exp-polyglot}). Where no benchmark exists, it supplies the search signal, and co-evolved agents outperform fixed-evaluator baselines~(\cref{sec:exp-writer}). Evaluator replacements act as a curriculum, re-ranking the archive toward new candidates around a resilient backbone~(\cref{sec:analysis}). Finally, co-evolution improves the evaluators themselves, and adversarial objectives correct failure modes such as self-preference~(\cref{sec:exp-adversarial}).

\subsection{Learned evaluation helps even where ground truth exists (RQ1)}
\label{sec:exp-polyglot}

In both coding domains, we co-evolve a code reviewer with the coder, providing a cheaper, complementary surrogate to the ground-truth signal. We also study reward design, where ground truth exists but non-stationary utilities can still improve search.

\noindent\textbf{\deepswe.}~~At low reasoning effort, the \ourmethod coder scores $82.1\%$ compared to $75.0\%$ for the best \hgmh coder~(\cref{fig:intro-token-efficiency}). At matched search tokens the two searches reach similar best-belief, but \hgmh never finds a better coder over the rest of its run, while \theourmethod continues to improve. It exceeds \astra at low ($68.8\%$) and medium ($76.8\%$) effort, and nearly matches it at high effort ($83.0\%$). These gains transfer to higher reasoning efforts: at medium and high effort, the \ourmethod coder passes $85.7\%$, ahead of the \hgmh coder ($78.6\%$, $82.1\%$)~(\cref{app:domain-protocols}). As on \polyglot, the co-evolved reviewer grades every coder patch, and this signal steers expansion~(\cref{app:evaluator-evolution:deepswe}).

\begin{figure}[t]
  \centering
  \includegraphics[width=\linewidth]{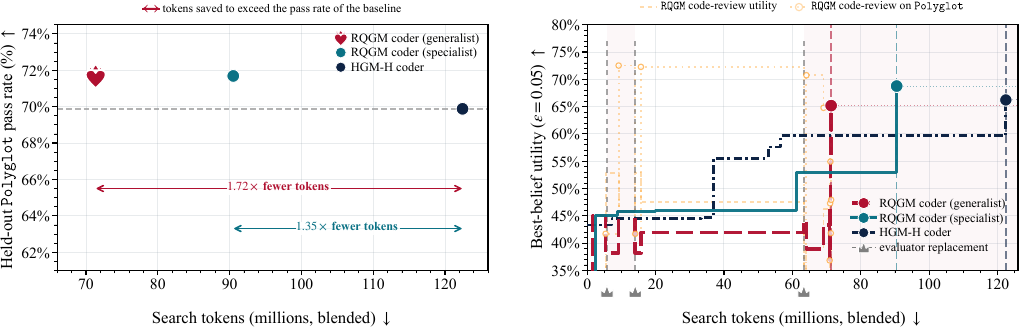}
  \caption{\textbf{\ourmethod\ exceeds the \hgmh baseline on \polyglot with $1.35\times$--$1.72\times$ fewer search tokens by adding a cheap evolved code reviewer.} \emph{Left:} held-out pass rate against search cost; the heart marks the best agent. \emph{Right:} best-belief utility across epochs. Crowns mark evaluator replacements, when erasure discards the displaced reviewer's utility records.}
  \label{fig:polyglot-token-efficiency}
\end{figure}

\noindent\textbf{\polyglot.}~~\Theourmethod exceeds the \hgmh held-out pass rate for both the specialist and generalist at $1.35\times$--$1.72\times$ lower token cost~(\cref{fig:polyglot-token-efficiency}). Examining where the meta-agent edits the codebase shows the benefits of co-evolution: in the \deepswe run, $83\%$ of accepted code patches modify shared task-agent functionality or infrastructure used by both the coder and the reviewer, rather than role-specific code, and in \polyglot $90\%$ do~(\cref{fig:analysis-transfer}). This suggests that co-evolution enriches the search rather than splitting its effort across competing objectives.

\begin{table}[t]
\centering
\caption{\textbf{A walking curriculum yields reward designers that transfer to jumping.} Five searches per arm at $125$ evaluations, six held-out trials each; mean $\pm$ std over $30$ trials.}
\label{tab:genesis_results}
\small
\setlength{\tabcolsep}{4.5pt}
\begin{tabular}{l c c c c c}
\toprule
Search & Held-out score $\uparrow$ & Median & Runs with jumps & Nodes/run & Search tokens (blended) \\
\midrule
Static evaluator & $0.306$\,{\color{black!55}\scriptsize$\pm 0.128$} & $0.302$ & $3/5$ & $22.0$ & $48.3$M \\
\bfseries \ourmethod curriculum & $\mathbf{0.365}$\,{\color{black!55}\scriptsize$\pm 0.053$} & $\mathbf{0.361}$ & $\mathbf{5/5}$ & $21.2$ & $53.8$M \\
\bottomrule
\end{tabular}
\end{table}

\noindent\textbf{Reward design.}~~Instead of changing the evaluator at epoch boundaries, \theourmethod changes the task dataset. Meta-agents design reward functions for reinforcement learning: search scores them on a quadruped learning to walk forward, while the held-out task is learning to jump~\citep{zhang2026hyperagents}. After a first epoch on the original walking task, \theourmethod adds a progressively larger obstacle in the robot's path~(\cref{app:reward-design}), and we compare it with the same search under a static evaluator. The held-out mean rises by $19\%$, and all five \ourmethod runs yield a jumping policy, against three of five for the static evaluator~(\cref{tab:genesis_results}). These results match those of~\hyperagents \citep{zhang2026hyperagents}, at $3.85$--$4.33\times$ lower cost, using \nemotronultra rather than \texttt{Claude Sonnet 4.5}.

\takeawaybox[RQ1:]{\Theourmethod enriches search utilities: even with ground truth, a co-evolved evaluator provides a complementary objective. On \deepswe it finds a stronger coder than fixed-evaluator search or base models, while on reward design it improves generalization from walking to jumping.}

\subsection{Learned evaluation improves agents without objective evaluation~(RQ2)}
\label{sec:exp-writer}

We now turn to domains with no objective benchmark, where only evaluator-dependent task agents exist. We score the artifacts produced by our paper writer and prover against a broad panel of reviewers and graders from both our work and prior baselines. Here we evaluate only the task agents, deferring their co-evolved evaluators to \cref{sec:exp-adversarial}, where anchor data permits further analysis.

\noindent\textbf{Paper writing.}~~\looseness=-1 We score writer artifacts against a panel of four reviewers~(\cref{tab:paper-writer-cross-reviewer}). At matched search tokens, \theourmethod writer achieves a $1.78\times$ higher average reviewer-panel acceptance rate than the \hgmh, and the best discovered agent across the full run beats the \hgmh with a $1.86\times$ higher average acceptance rate. We ascribe this improvement to the regularized search, which favors reviewers that are harsher on \AI-generated text~(\cref{sec:exp-adversarial}).

\begin{table}[H]
\centering
\caption{\textbf{Co-evolved writers achieve $1.78\times$ higher mean acceptance at matched search cost, and $1.86\times$ for the best-found specialist.} The four \textbf{Reviewers} give the percentage of writer-generated papers that the specific reviewer accepts (mean $\pm$ 95\% Jeffreys half-width). \textbf{\ourmethod adversarial} is the \ourmethod reviewer regularized to be harsher on \AI-written text (\cref{sec:exp-adversarial}). Best per column in \textbf{bold}.}
\label{tab:paper-writer-cross-reviewer}
\small
\setlength{\tabcolsep}{3.5pt}
\renewcommand{\arraystretch}{1.12}
\newcommand{\rvhead}[2]{\parbox[t]{#1}{\centering\textbf{#2}\par}}
\begin{tabular}{lrcccc@{\hskip 8pt}r}
\toprule
& & \multicolumn{4}{c}{\textbf{Reviewer Acceptance Rate~(\%,~$\uparrow$)}} & \\
\cmidrule(lr){3-6}
\textbf{Writer} & \parbox[t]{30pt}{\raggedleft\textbf{Search}\\\textbf{Tokens}} & \rvhead{37pt}{\sakana} & \rvhead{37pt}{\dgmh} & \rvhead{47pt}{\hgmh} & \rvhead{44pt}{\ourmethod\\adversarial} & \makebox[6.8em][c]{\textbf{Mean~(\%,~$\uparrow$)}} \\
\midrule
\hgmh writer & 42.5M & $1.0$\,{\color{black!55}\scriptsize$\pm 2.2$} & $12.0$\,{\color{black!55}\scriptsize$\pm 6.3$} & $64.0$\,{\color{black!55}\scriptsize$\pm 9.3$} & $10.0$\,{\color{black!55}\scriptsize$\pm 5.9$} & \makebox[6.8em][l]{\makebox[2.5em][r]{$21.8$}\,{\color{black!55}\scriptsize$\pm 4.0$}\,\makebox[2.6em][l]{}} \\
\addlinespace[2pt]
\ourmethod writer (generalist) & 44.6M & $2.0$\,{\color{black!55}\scriptsize$\pm 2.9$} & \textbf{43.0}\,{\color{black!55}\scriptsize$\pm 9.6$} & $81.0$\,{\color{black!55}\scriptsize$\pm 7.6$} & $29.0$\,{\color{black!55}\scriptsize$\pm 8.8$} & \makebox[6.8em][l]{\makebox[2.5em][r]{$38.8$}\,{\color{black!55}\scriptsize$\pm 4.8$}\,\makebox[2.6em][l]{\scriptsize($1.78\times$)}} \\
\addlinespace[2pt]
\ourmethod writer (specialist) & 221.8M & \textbf{5.0}\,{\color{black!55}\scriptsize$\pm 4.3$} & $40.0$\,{\color{black!55}\scriptsize$\pm 9.5$} & \textbf{86.0}\,{\color{black!55}\scriptsize$\pm 6.8$} & \textbf{31.0}\,{\color{black!55}\scriptsize$\pm 9.0$} & \makebox[6.8em][l]{\makebox[2.5em][r]{\textbf{40.5}}\,{\color{black!55}\scriptsize$\pm 4.8$}\,\makebox[2.6em][l]{\scriptsize($1.86\times$)}} \\
\bottomrule
\end{tabular}
\end{table}

\begin{table}[htb]
\centering
\caption{\textbf{The co-evolved \ourmethod specialist prover attains the best mean score.} A panel of three graders scores every proof. We pool each metric across graders, with per-grader results in \cref{tab:proof-per-grader}; without the \ourmethod grader, the specialist still matches \imocode ($4.03$ vs.\ $4.08$) and beats \hgmh ($3.60$). \textbf{Score} is the mean grade ($0$--$7$), \textbf{\texttt{Pass@6}} the fraction scoring at least $6$ of $7$, and \textbf{\texttt{Pass@7}} the fraction earning full credit; all report $\pm$ s.e.m. Best per column in \textbf{bold}.}
\label{tab:imo-proof-cross-grader}
\small
\setlength{\tabcolsep}{5pt}
\renewcommand{\arraystretch}{0.95}
\begin{tabular}{lrccc}
\toprule
\textbf{Prover} & \textbf{Search tokens} & \textbf{Score~($\uparrow$)} & \textbf{\texttt{Pass@6}~($\uparrow$)} & \textbf{\texttt{Pass@7}~($\uparrow$)} \\
\midrule
Static \imocode prover & N/A & $4.07$\,{\color{black!55}\scriptsize$\pm 0.42$} & $55.0\%$\,{\color{black!55}\scriptsize$\pm 6.4$} & \textbf{55.0\%}\,{\color{black!55}\scriptsize$\pm 6.4$} \\
\hgmh prover & 21.5M & $3.73$\,{\color{black!55}\scriptsize$\pm 0.43$} & $51.7\%$\,{\color{black!55}\scriptsize$\pm 6.5$} & $45.0\%$\,{\color{black!55}\scriptsize$\pm 6.4$} \\
\ourmethod prover (generalist) & 37.9M & $3.73$\,{\color{black!55}\scriptsize$\pm 0.43$} & $51.7\%$\,{\color{black!55}\scriptsize$\pm 6.5$} & $45.0\%$\,{\color{black!55}\scriptsize$\pm 6.4$} \\
\ourmethod prover (specialist) & 88.0M & \textbf{4.33}\,{\color{black!55}\scriptsize$\pm 0.41$} & \textbf{61.7\%}\,{\color{black!55}\scriptsize$\pm 6.3$} & $48.3\%$\,{\color{black!55}\scriptsize$\pm 6.5$} \\
\bottomrule
\end{tabular}
\end{table}

\noindent\textbf{Proof writing.}~~As shown in \cref{tab:imo-proof-cross-grader}, \theourmethod generalist performs similarly to the \hgmh at matched search tokens. Divergence appears at longer horizons where the \hgmh stagnates. \Theourmethod specialist prover achieves the best panel-mean score and \texttt{Pass@6} score; without our own grader in the panel, it still matches the static \imocode prover ($4.03$ vs.\ $4.08$) and beats \hgmh ($3.60$). As we show in \cref{sec:exp-adversarial}, this is enabled by a co-evolved grader that exceeds \proofautograder and \hgmh.

A more nuanced comparison is the \imocode baseline of \citet{huang2025winninggoldimo2025}, a human-engineered verification-and-refinement pipeline that achieved gold-medal performance at IMO 2025. \Theourmethod specialist attains a higher panel-mean score despite using no hand-engineered Olympiad-specific scaffold, but achieves this by finding more near-complete proofs ($6$ of $7$). Details appear in \cref{app:extended-results}.

\takeawaybox[RQ2:]{Co-evolving agents with learned evaluators outperform fixed-evaluator baselines: co-evolved writers reach markedly higher reviewer-panel acceptance than fixed-evaluator writers.}

\subsection{Evaluator replacements act as a curriculum on the search (RQ3)}
\label{sec:analysis}

A mechanism by which \theourmethod can improve task agents is that a progressively harsher evaluator imposes a curriculum-like effect on the task agents. For this, each evaluator transition must re-rank the search under the stricter criterion while preserving strong lineages, so that the search advances.

\begin{figure}[htb]
    \centering
    \includegraphics[width=\linewidth]{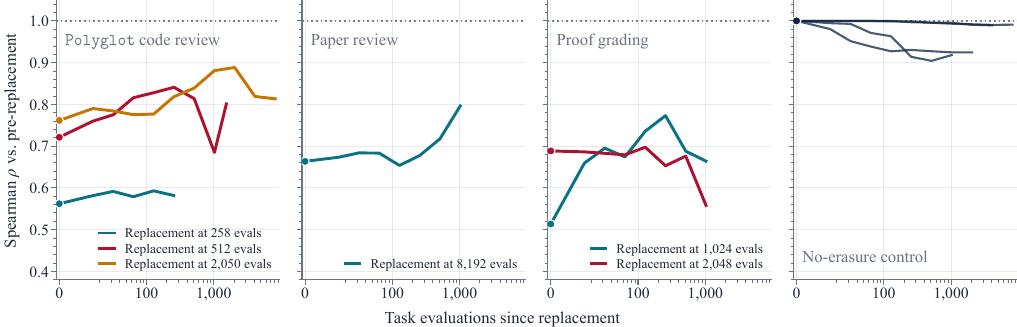}
   \caption{\textbf{Evaluator replacements permanently re-rank the archive.} Each curve tracks one replacement, plotting the Spearman $\rho$ between post- and pre-replacement rankings as evaluations accumulate: $\rho = 1$ (dotted) is the unchanged order, $\rho = 0$ indicates an uncorrelated reordering. Across all three tasks $\rho$ settles well below $1$ and never recovers. The no-erasure control (rightmost) stays high ($\rho \geq 0.90$), showing erasure is necessary for utility transitions to guide search.}
    \label{fig:transition-rerank}
\end{figure}

\noindent\textbf{Macro view: each replacement re-ranks the archive.}~~\looseness=-1 If replacements merely sharpened pass-rate estimates without changing which agents the search favors, no curriculum could arise. \Cref{fig:transition-rerank} shows that part of the ranking persists across a transition, since evaluator-independent evidence carries through, but the re-ordering never recovers. Without erasure the ranking remains highly stable.

\begin{figure}[htb]
    \centering\includegraphics[width=\linewidth]{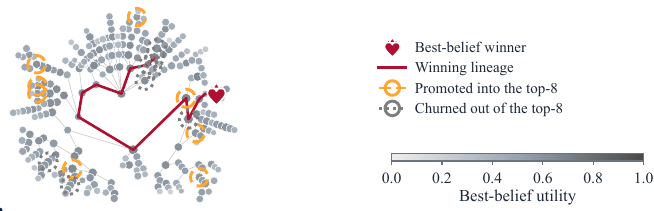}
\caption{\textbf{Evaluator replacement preserves the best lineage.} The paper archive after replacement (node color is best-belief utility, size tracks evidence count). The crimson winning lineage survives. Among the top-$8$ nodes, every prior member is re-ranked (rings mark promotions and churn).}
    \label{fig:tree-paper}
\end{figure}

\noindent\textbf{Micro view: the curriculum advances a backbone lineage.}~~\Cref{fig:tree-paper} shows that the archive concentrates evidence on a few backbone lineages. While candidates are often re-ranked, the best lineage stays resilient, suggesting the curriculum improves candidates around a stable backbone.

\takeawaybox[RQ3:]{Evaluator replacements drive a curriculum-like search: erasure lets each stricter criterion re-rank the archive toward new candidates, while a resilient backbone keeps the population progressing.}

\subsection{Co-evolution improves the evaluators themselves (RQ4)}
\label{sec:exp-adversarial}
\begin{figure}[htb]
    \centering
\includegraphics[width=\linewidth]{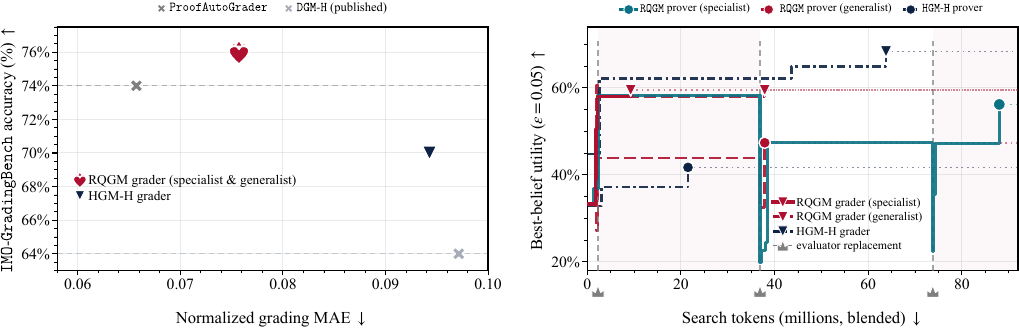}
\caption{\textbf{The co-evolved \ourmethod\ grader reaches the best \imogradbench\ accuracy at $3\times$ lower search cost than \hgmh.} \emph{Left:} \imogradbench\ accuracy of selected graders against grading mean absolute error (MAE). \emph{Right:} best-belief utility over tokens for graders (triangles) and provers (circles). Utility drops at each evaluator replacement (crown), with shading marking epochs, as erasure discards the displaced records, then re-climbs.}
    \label{fig:prover-grader-results}
\end{figure}

\noindent\textbf{Proof grading.}~~As shown in \cref{fig:prover-grader-results}, the co-evolved grader is the strongest judge, outperforming the static baselines and the best \hgmh grader at $3\times$ lower search cost; the specialist and generalist coincide in a single node. This better grader drives the prover gains of \cref{tab:imo-proof-cross-grader}. Similarly, on \deepswe, only the top reviewer rejects most failing patches~(\cref{tab:deepswe-reviewer-fidelity}) with the others being overly lenient.

\begin{figure}[htb]
    \centering
    \includegraphics[width=\linewidth]{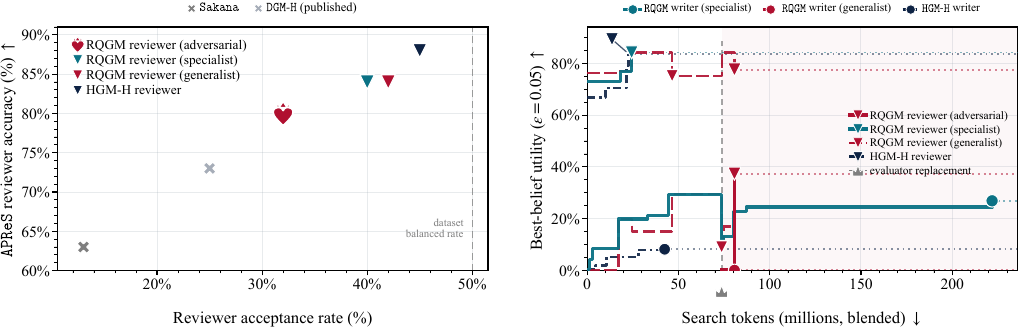}
    \caption{\textbf{The adversarial \ourmethod\ reviewer accepts \AI{} and human papers at similar rates.} \emph{Left:} \apres\ accuracy of selected reviewers against acceptance rate; the dashed line marks the dataset's $50\%$ accept rate; its $32\%$ acceptance matches its $29$--$31\%$ on writer papers~(\cref{tab:paper-writer-cross-reviewer}). \emph{Right:} the same run over tokens, showing best-belief utility for reviewers (triangles) and writers (circles); at the adversarial-pool replacement the generalist's average best-belief drops as erasure re-ranks the affected utilities, then re-climbs under the harsher criterion.}
    \label{fig:paper-writer-reviewer-results}
\end{figure}

\noindent\textbf{Paper review.}~~\looseness=-1 The \hgmh reviewer scoring writers in \cref{tab:paper-writer-cross-reviewer} accepts \AI-generated papers at $1.42\times$--$1.91\times$ the rate of human ones. After each evaluator replacement in \theourmethod, the \AI-generated papers the displaced reviewer had accepted form an adversarial pool, and the next epoch additionally rewards evolved reviewers for rejecting them. As shown in \cref{fig:paper-writer-reviewer-results}, the search then selects a reviewer that is harsh specifically on \AI-generated text, which costs raw accuracy by construction and leaves our best-found reviewer below the lenient \hgmh. However, this calibrates the reviewer to accept \AI-generated and human-written papers at similar rates while retaining $80\%$ ground-truth accuracy.

\takeawaybox[RQ4:]{Co-evolution finds a stronger grader at lower cost than fixed-evaluator baselines and, by introducing an adversarial objective at an epoch boundary, corrects a reviewer's self-preference bias.}

\relax

\section{Conclusion}
\label{sec:conclusion}
We introduced the \textbf{Red Queen G\"odel Machine (\ourmethod)}, a framework in which evaluators co-evolve with the agents they score. Through controlled utility evolution, the search objective can change at epoch boundaries while remaining fixed within each epoch, which has three consequences. First, where ground truth exists, a co-evolved evaluator adds a complementary search signal, and co-evolved coders outperform the \hgmh baseline on both coding benchmarks. Second, co-evolving the evaluator alongside the generator enables improvement on hard-to-verify tasks such as paper and proof writing, where a fixed benchmark cannot score the artifact directly. Third, the search can adopt interventions that a fixed objective cannot express, such as a curriculum or an adversarial term. At longer horizons, we believe these mechanisms could help push the frontier of AI capabilities.

\FloatBarrier
\label{paper:main-end}
\clearpage
\ifpapericlr\setcounter{totalnumber}{3}\fi
\ifpaperanon\else
\section*{Acknowledgments}

This research was supported by the following entities: NVIDIA, The Royal Academy of Engineering via DANTE (a RAEng Chair); the European Research Council, specifically the REDIAL project; SPRIND under the composite learning challenge; Foresight Institute via the AI Safety Grant; Google through a Google Academic Research Award. We also thank Mara Darin Agha for advising on the narrative of this work and its related artifacts, and for constructing their visual language and design.\fi
\section*{AI Use Statement}
\label{sec:ai-use}

We used \texttt{GPT-5.5}, \texttt{GPT-5.6 Sol}, and \texttt{GPT-6 Astra} in the experimental work reported in this paper. \nemotronultra was used for the reward-design experiment, and \nemotron and \nemotronultra for ablations present in the Appendix.

We used \texttt{GPT-5.5} and \texttt{GPT-5.6 Sol} to polish writing, abbreviate or rephrase text, evaluate its clarity, and identify areas for improvement. These models also assisted with discovery of additional papers for our extended literature review in \cref{app:related}, building on the authors' initial review.

We used \texttt{GPT-5.5} and \texttt{GPT-6 Astra} to generate plotting code and as general code assistants. Finally, we used \texttt{GPT-6 Astra} and \texttt{GPT-5.6 Sol} to assist with experiments and check proofs.

\section*{Ethics statement}
This work automates the writing and reviewing of scientific papers, which carries a dual-use risk: automated writers could be used to flood peer review with generated submissions, and automated reviewers could be deployed in place of human judgment. No generated paper was submitted to any venue or shown to human reviewers, and all reviewing experiments use the public \apres dataset. Our results bear on this risk directly: fixed LLM reviewers accept \AI-generated papers more readily than human ones, and \theourmethod's adversarial objective produces a reviewer that is equally stringent on both. We view calibrated automated evaluation as a safeguard for peer review, not a replacement for it. The work involves no human subjects or personal data.

\section*{Reproducibility statement}
All implementation details needed to reproduce our results are provided in~\cref{app:algo}, along with the experimental setup in~\cref{app:exp-details}. Additional results and ablations are in~\cref{app:extended-results}. The supplementary zip contains the anonymized code together with a README and scripts that give step-by-step instructions to install the system, check the installation on a toy domain, and run both \ourmethod and the \hgmh baseline in every domai.

\bibliographystyle{unsrtnat}
\bibliography{bib/references}

\clearpage
\appendix
\crefalias{section}{appendix}
\crefalias{subsection}{appendix}
\crefalias{subsubsection}{appendix}
\etocdepthtag.toc{appendix}
\clearpage
\flushbottom
\renewcommand{\topfraction}{0.9}\renewcommand{\bottomfraction}{0.8}\renewcommand{\textfraction}{0.07}
\renewcommand{\floatpagefraction}{0.85}\setcounter{topnumber}{3}\setcounter{bottomnumber}{2}\setcounter{totalnumber}{5}
\emergencystretch=1.5em
\phantomsection
\label{sec:appendix}
\etocdepthtag.toc{main}
\etocdepthtag.toc{appendix}

\vspace{0.8cm}
\begingroup
  \hypersetup{linkcolor=black}
  \color{black}
  \etocsettagdepth{main}{none}
  \etocsettagdepth{appendix}{subsection}
  \etocsettocstyle{\centering\bfseries {\Large Table of Contents}\par
    \vspace{0.1\baselineskip}
    \hrule
    \vspace{0.4\baselineskip}
  }{\vspace{0.4\baselineskip}
    \hrule height 0.4pt
  }
  \tableofcontents
\endgroup
\vspace{0.8\baselineskip}

\section{Appendix Overview}
The appendix gives the full algorithm~(\cref{app:algo}), experimental setup~(\cref{app:exp-details}), extended results~(\cref{app:extended-results}), what the co-evolution changed in the code~(\cref{app:evaluator-evolution}), theory~(\cref{app:proofs}), limitations~(\cref{app:limitations}), and extended related work~(\cref{app:related}).

\section{Method Details and Algorithm}
\label{app:algo}

\subsection{Search Space}
\label{sec:method:multiagent}

\noindent\textbf{Roles and task pools.}~~Each node is a workspace holding $K$ roles $\mathcal{R}$ made up of task agents and evaluators, plus a meta-agent that edits the workspace~\citep{zhang2026hyperagents}. Each role $r$ is scored on its own task pool $\mathcal{D}_r$; a paper reviewer, for example, has one task for human-written and one for agent-generated papers. Since each task belongs to one role, we follow \hyperagents in calling every role's agent a \emph{task agent} when the role is clear.

\noindent\textbf{Evaluator slots and epochs.}~~An evaluator-dependent role is scored through an \emph{evaluator slot} whose \emph{epoch index} counts its replacements. The $M\le K$ slots are replaced independently, so their indices form an epoch vector $\boldsymbol{j}\in\mathbb{N}^{M}$, and replacing one slot advances one entry.

\noindent\textbf{Epoch-local utility.}~~Writing $p_{r,d,\boldsymbol{j}}(a)$ for node $a$'s success rate on task $d$ under epoch vector $\boldsymbol{j}$, the utility of \cref{sec:method:formal} becomes $U_{\boldsymbol{j}}(a)=\frac{1}{K}\sum_{r}\frac{1}{|\mathcal{D}_r|}\sum_{d\in\mathcal{D}_r}p_{r,d,\boldsymbol{j}}(a)$. We average uniformly rather than by sample count because some tasks can generate unboundedly many evaluations.

\noindent\textbf{Data isolation.}\label{sec:method:train}~~Feedback that \emph{creates} nodes is kept apart from evidence that \emph{selects} them~(\cref{sec:experimental-design}). For evaluator-independent roles, the splits of \cref{app:splits} do this directly. For evaluator-dependent roles, the task agent generates fresh artifacts for training and for search, and only the search artifacts, scored by the epoch's frozen evaluator, enter the utility.

\subsection{Three-Level Sampling Hierarchy}
\label{sec:method:hierarchy}

The scheduler of \citet{wang2025huxley} interleaves expansion and evaluation. With $N_t$ evaluations performed and expansion exponent $\alpha\in(0,1)$, a UCB-Air gate~\citep{NIPS200849ae49a2} expands when $N_t^{\alpha}\ge|\mathcal{T}_t|$ and evaluates otherwise, so the archive grows as $\mathcal{O}(N_t^{\alpha})$. An expansion draws its parent $a^\star$ by Thompson sampling over \cmp. An evaluation draws $a^\star$ over its own counts, then descends two levels: the least-evaluated eligible role $r^\star=\argmin_{r}n_r(a^\star)$, then its least-evaluated eligible task $d^\star=\argmin_{d\in\mathcal{D}_{r^\star}}n_d(a^\star)$. A task is eligible unless its dataset is exhausted; a role is eligible if any of its tasks is. Outcomes accumulate at the node, and the role and task counters only balance load. Under this balancing, the pooled Beta posterior mean converges to $U_{\boldsymbol{j}}(a)$~(\cref{prop:balanced-working-posterior}).

\subsection{Co-Evolving Evaluation and Controlled Utility Evolution}
\label{sec:method:generative}
\label{sec:method:utility}

\Cref{sec:method:transitions} states the replacement rule, selective erasure, and checkpoint schedule. Formally, a \emph{utility transition}~(\cref{def:utility-transition}) replaces a slot's evaluator and erases only the records that depended on it; evaluator replacement is one instance, and time-varying benchmarks or adversarial objectives fit the same structure. Old nodes are re-scored only when later evaluations return to them, and the sampling hierarchy above balances roles and tasks within each selected node. Exponentially spaced checkpoints keep the total record reprocessing linear in the budget~(\cref{prop:linear-work}) and let later evaluators, each promoted on its own anchor, steer the rest of the search.

\subsection{The Algorithm in Full}

\Cref{alg:main} shows \theourmethod further visualized in \cref{fig:method-expanded}. Lines 1--3 freeze every evaluator slot on its incumbent; training is demand-driven, so a node is train-evaluated only when first chosen for expansion. Lines 4--22 run the \cmp-guided search between checkpoints: the gate (line 8) either expands a parent drawn by metaproductivity into a child (lines 9--16), or spends one validation evaluation on the least-measured role--task cell of a node (lines 18--21). Lines 23--32 run only at checkpoints: each slot promotes the candidate with the highest anchor best-belief, ties to the incumbent, and on a change advances its epoch and erases the displaced evaluator's records. Lines 34--36 return highest best-belief node.

\begingroup
\makeatletter
\ifpapericlr\IncMargin{1.5em}\fi
\ifpaper@twocolumn
  \begin{algorithm*}[t]
\else
  \begin{algorithm}[t]
\fi
\DontPrintSemicolon
\LinesNumbered
\SetAlgoLined
\SetNlSty{textbf}{\scriptsize}{}
\SetAlgoNlRelativeSize{-1}
\footnotesize
\KwIn{Seed workspace $a_0$; roles $\mathcal R$; validation budget $B$; evaluator slots $1{:}M$}
\KwIn{Anchor sets $\mathrm{GT}[m]$; optional adversarial pools $P[m]\leftarrow\emptyset$; checkpoints $\mathcal C$; caps $(\bar B_{\rm exp},\bar B_{\rm train},\bar B_{\rm val})$; best-belief level $\epsilon$}
\KwOut{Archive $\mathcal T$ and selected agent $a^\star$}
\BlankLine
\textbf{Initialize the frozen-epoch archive.}\;
$\mathcal T \leftarrow \{a_0\}$; $N \leftarrow 0$; $j[m]\leftarrow 1$ for all slots $m$\;
$E[m]\leftarrow \func{Freeze}(\func{Incumbent}(m))$ for all slots $m$\;
\BlankLine
\textbf{Run CMP-guided search between checkpoints.}\;
\While{$N < B$}{
  $b \leftarrow \min\{c\in\mathcal C:c>N\}\wedge B$\tcp*{next checkpoint}
  \While{$N < b$}{
    \uIf{$\func{ExpandGate}(\mathcal T,N)$}{
      $a_p \leftarrow \func{SampleCladeByCMP}(\mathcal T)$\;
      \If{$\neg\func{TrainedEnough}(a_p)$}{
        $\func{TrainEval}(a_p,\mathcal R,E,\bar B_{\rm train})$\tcp*{demand-driven; lineage only}
      }
      $a' \leftarrow \func{EditWorkspace}(a_p,\mathcal T,\bar B_{\rm exp})$\;
      \If{$\func{Valid}(a')$}{
        $\mathcal T \leftarrow \mathcal T \cup \{a'\}$\;
      }
    }
    \Else{
      $a \leftarrow \func{SampleNodeByCounts}(\mathcal T)$\tcp*{node-level \hgm choice}
      $(r,d) \leftarrow \func{LeastMeasuredCell}(a,\mathcal R)$\tcp*{role/task balancing}
      $o \leftarrow \func{ValEval}(a,r,d,E,P,\bar B_{\rm val})$\;
      $\func{AddUtilityRecord}(\mathcal T,a,r,d,o,\mathbf j)$; $N \leftarrow N+1$\;
    }
  }
\BlankLine
  \textbf{At checkpoint, replace evaluators by anchor (and pool) best-belief.}\;
  \ForEach{evaluator slot $m$}{
    $\Gamma_m \leftarrow \{E[m]\}\cup \func{Challengers}(m,\mathcal T)$\;
    $\mathrm{BB}[e] \leftarrow \func{SlotBB}_\epsilon(e;\mathrm{GT}[m],P[m])$ for each $e\in\Gamma_m$\tcp*{anchor-only if $P[m]=\emptyset$}
    $e^\star \leftarrow \func{Best}(\Gamma_m,\mathrm{BB};\ \func{tie}=E[m])$\;
    \If{$e^\star \neq E[m]$}{
      \If{$m$ uses an adversarial pool}{
        $P[m]\leftarrow\func{AcceptedGenerations}(E[m])$\tcp*{displaced evaluator's accepts}
      }
      $j[m]\leftarrow j[m]+1$; $E[m]\leftarrow \func{Freeze}(e^\star)$\;
      $\func{EraseSlotRecords}(\mathcal T,m)$; $\func{RecomputeCMPStats}(\mathcal T)$\;
    }
  }
}
\BlankLine
\textbf{Select the final agent by best-belief over archive nodes.}\;
$a^\star \leftarrow \func{BestBeliefNode}(\mathcal T,\epsilon)$\;
\Return{$(\mathcal T,a^\star)$}\;
\caption{\small Red Queen G\"odel Machine search with anchor-gated evaluator replacement and an optional adversarial pool.}
\label{alg:main}
\ifpaper@twocolumn
  \end{algorithm*}
\else
  \end{algorithm}
\fi
\makeatother
\endgroup

\begin{widefigure}[htbp]
\centering
\includegraphics[width=\linewidth]{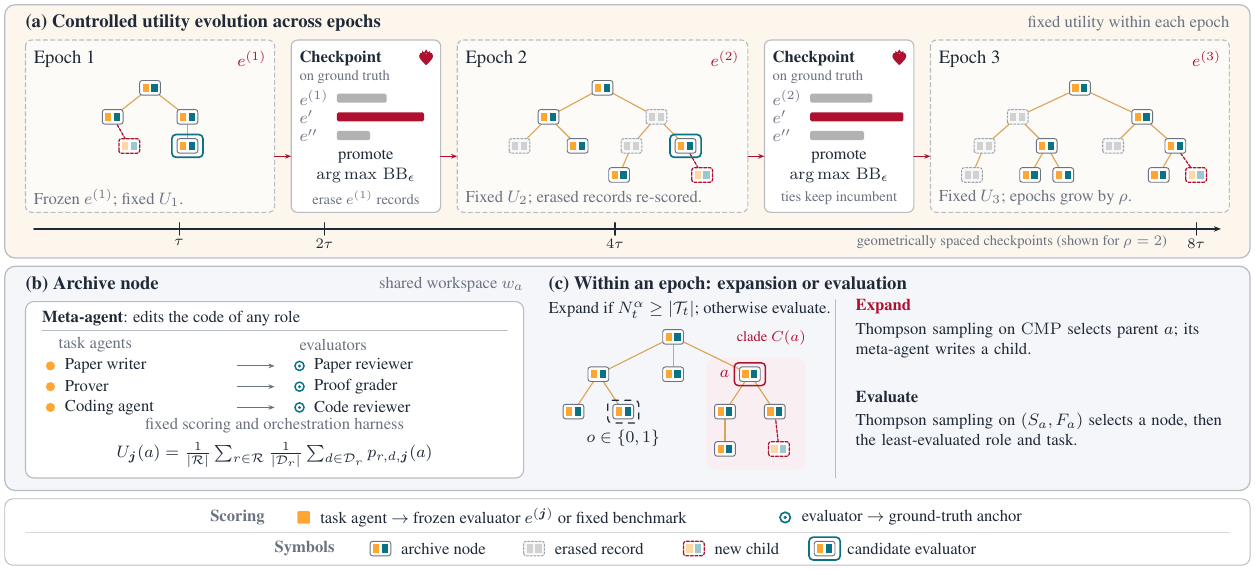}
\caption{\textbf{Expanded view of \ourmethod.} (a) Evaluators remain frozen within each epoch; geometrically spaced checkpoints compare candidates on ground-truth anchors, and replacement erases dependent records. (b) Each archive node contains a shared workspace of task agents, evaluators, and a meta-agent. (c) Within an epoch, search alternates expansion and evaluation~(\cref{alg:main}).}
\label{fig:method-expanded}
\end{widefigure}
\FloatBarrier

\subsection{Continuous Utilities}
\label{app:continuous}

Reward design scores each evaluation in $[0,1]$ rather than as pass or fail. To accommodate this, we modify the \cmp to be the average utility over a node's subtree. The Beta posteriors behind Thompson sampling and best-belief become Student-$t$ posteriors over the mean utility, the exact marginal of a Normal--Gamma prior when the variance is unknown.

\section{Experimental Setup}
\label{app:exp-details}

\subsection{Visualization Guide}
\label{exp_design:reading_results_figures}
Each result figure pairs a \emph{ground-truth panel}, which scores arms such as the \ourmethod generalist, specialist, and \hgmh on a fixed anchor, with a \emph{search-trajectory panel} of best-belief utility over search tokens. Circles mark task agents, down-triangles learned evaluators, and the crowned heart the selected best-belief winner. Trajectory heights are measured against each arm's \emph{own} evaluator, so only the ground-truth panel compares arms. A crown on a dashed rule marks an evaluator replacement, and shaded bands mark epochs; best-belief may drop at a replacement, as erasure discards the displaced evaluator's records, and a value of $0$ marks an epoch with no surviving record yet. A trajectory ends where its best-belief agent stops changing, and a dotted line continues it to the end of the matched budget. When the specialist and generalist are the same node, as for the \deepswe coder, we draw a single line. Arrows give each metric's preferred direction.

\subsection{Domains and Anchor Pairs}

\Cref{fig:domain-utilities} shows how each domain pairs a generator with a learned evaluator and a ground-truth anchor~(\cref{sec:experimental-design}). In the \polyglot run, the co-evolved code reviewer scores every coder patch as it is generated, alongside the benchmark's executable tests.

\begin{widefigure}[tbp]
\centering

\includegraphics[width=\linewidth]{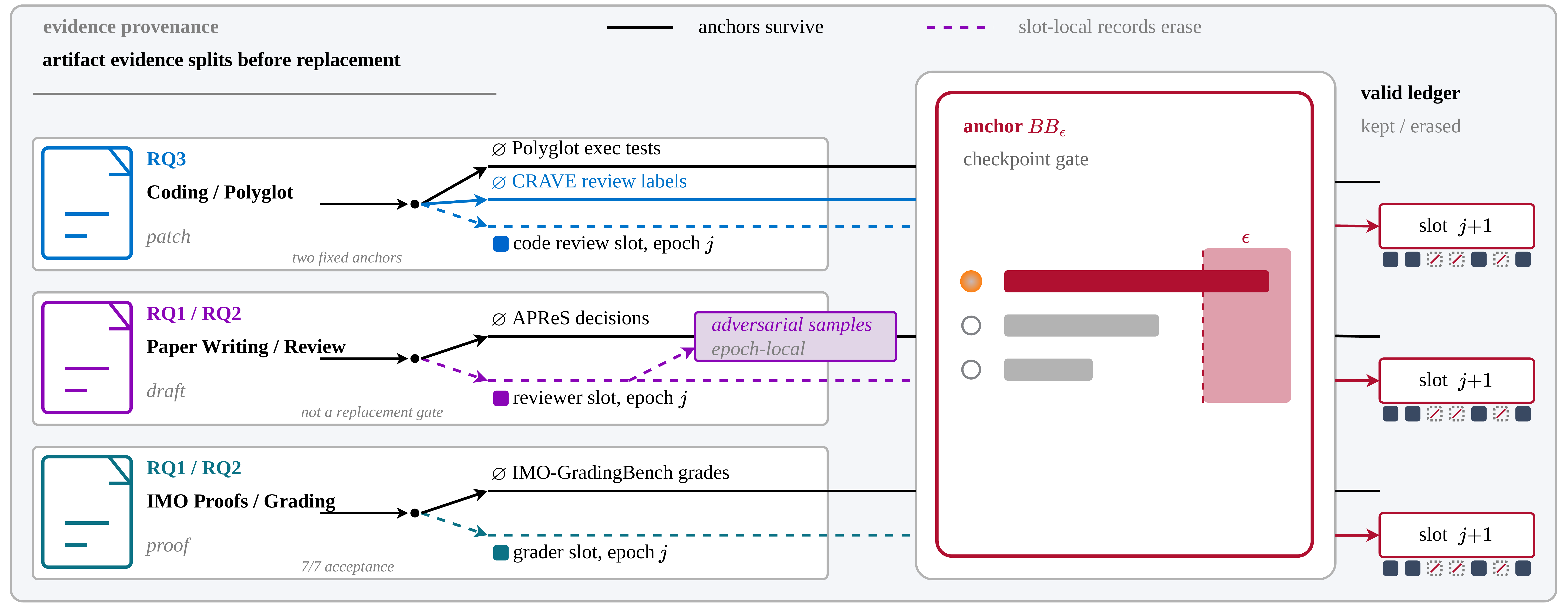}
\caption{\textbf{Each domain separates durable anchor evidence from epoch-local evaluator records.} \emph{Left:} fixed ground-truth anchors (solid lines) survive evaluator replacement, while epoch-frozen evaluator-slot records (dashed lines) are local to epoch $j_m$. \emph{Middle:} only evaluator-independent anchor evidence enters the $\epsilon$-best-belief evaluator replacement gate. \emph{Right:} after replacement, anchors remain valid while records from the displaced evaluator slot are erased before epoch $j_m+1$.}
\label{fig:domain-utilities}

\end{widefigure}

\subsection{Baselines and Agent Selection}
\label{app:baselines-selection}
\label{app:search-interventions}
\label{app:selection-costs}

Because \theourmethod optimizes several utilities at once, we report two best-belief agents per run (if they dot coincide on the same node): the \emph{specialist}, best on one target utility, and the \emph{generalist}, best on their average~(\cref{sec:experimental-design}). The baselines and adversarial interventions are those of \cref{sec:experimental-design}; we evaluate the available published \hyperagents patches in the paper and proof domains. Headline writer-panel uncertainties are $95\%$ Jeffreys intervals and proof-panel uncertainties standard errors~(\cref{app:domain-protocols}).

\noindent\textbf{Proof mechanism ablation.}~~The control is coupled \hgm: writer and grader code may both evolve, while the search evaluator stays fixed and the adversarial pool is disabled.

\subsection{Models and Run Identities}
\label{app:fm}

The headline \polyglot, paper, and proof runs, their held-out evaluations, and all panel-scoring calls use \gptlow (low); the \deepswe runs use \astra~(\cref{tab:deepswe-settings}). The cost ablation of \cref{sec:reducing_search_costs} switches only the search-time paper task agents to \nemotronultra, in addition to \nemotron~\citep{nemotron_super} appears only in the mechanism ablations~(\cref{app:nemotron-ablations}). Reported token totals include train-time calls. There are eight headline runs: an \ourmethod and an \hgmh run in each of the coding, paper, and proof domains, where one paper \ourmethod run yields the writer, reviewer, and adversarial reviewer, and one proof run yields the prover and grader. All runs share one total budget, and each agent is read where its search reaches its best-belief selection. The proof-domain mechanism ablation uses \nemotronultra at medium reasoning for task agents, meta-agents, and downstream graders. It uses a $1{,}024$-evaluation budget to bound the total cost.

\subsection{Search and Scheduling Settings}
\label{app:hyperparams}

Replacement and final selection share the $\epsilon$-best-belief rule of \cref{app:swap-rule}, applied to nodes with at least $5$ evaluations. \Cref{tab:settings-fixed,tab:settings-editable} list the fixed and meta-agent-editable settings; everything else follows \citet{wang2025huxley}. \citet{wang2025huxley} use $8{,}000$ evaluations for one task.

\paragraph{\deepswe expansion warmup.}
The \deepswe runs first build a warmup population, which is then exploited: the expansion exponent is $\alpha=2/3$ until $600$ search evaluations and $\alpha=1/2$ afterwards, on the same cumulative counter, keeping every warmup node and outcome.

\begin{table}[tbp]
\centering
\caption{\textbf{Fixed search configuration shared by the eight headline runs.} These constants are held fixed across every run and are not exposed to the meta-agent. No search cool-down is used.}
\label{tab:settings-fixed}
\small
\begin{tabular}{@{}ll@{}}
\toprule
Setting & Value \\
\midrule
\multicolumn{2}{@{}l}{\textit{Models and budget}}\\
Base model & \gptlow (low) \\
Ablation model & \nemotron \\
Total budget per run & $12{,}288$ evaluations \\
Train samples per node & $3$ \\
\addlinespace
\multicolumn{2}{@{}l}{\textit{Selection and scheduling}}\\
Best-belief quantile & $\epsilon=0.05$, five-outcome anchor minimum \\
Checkpoint schedule & power-of-two ($\rho=2$) \\
UCB-Air expansion exponent & $\alpha=0.6$ (inherited, \citealp{wang2025huxley}) \\
Expansion gate & expand when $N_t^{\alpha}\geq|\mathcal{T}_t|$ \\
Exploration--exploitation scheduler & $B/b$ ($b$ = remaining budget) \\
\addlinespace
\multicolumn{2}{@{}l}{\textit{Data splits}}\\
Test sets (\apres / \imogradbench / \crave) & $100$ items each \\
\polyglot train split & $10$ items \\
\polyglot validation split & $49$ items \\
\polyglot test split & $166$ items \\
\addlinespace
\multicolumn{2}{@{}l}{\textit{Budget caps}}\\
Expand cap & \$25, $1{,}200$\,s \\
Train cap & \$8, $900$\,s \\
Validation cap & \$25, $1{,}200$\,s \\
\bottomrule
\end{tabular}
\end{table}

\begin{table}[tbp]
\centering
\caption{\textbf{Meta-agent-editable surface and its initial values.} The meta-agent may change any.}
\label{tab:settings-editable}
\small
\begin{tabular}{@{}ll@{}}
\toprule
Setting & Initial value \\
\midrule
Output tokens per call & $32{,}768$  \\
Timeout, run & $86{,}400$\,s \\
Timeout, eval & $1{,}200$\,s \\
Timeout, LLM & $300$\,s \\
Timeout, shell & $120$\,s \\
Coder tool calls per task & $\leq 16$ \\
Meta-agent tools & persistent bash ($120$\,s); editor \\
Meta-agent tool calls per turn & $\leq 40$ \\
\bottomrule
\end{tabular}
\end{table}

\Cref{tab:deepswe-settings} lists the configuration of the two \deepswe runs.

\begin{table}[tbp]
\centering
\caption{\textbf{Search configuration of the two \deepswe runs.} Shared by both runs unless noted. Runs are compared on search tokens, not evaluation counts; the winners are the coders of \cref{fig:intro-token-efficiency}.}
\label{tab:deepswe-settings}
\small
\begin{tabular}{@{}ll@{}}
\toprule
Setting & Value \\
\midrule
\multicolumn{2}{@{}l}{\textit{Models and budget}}\\
Meta-agent / task agents & \astra (high) / \astra (low) \\
Search tokens to winner & $587$M (\ourmethod), $429$M (\hgmh) \\
Measured train results per expansion & $\leq 16$ \\
\addlinespace
\multicolumn{2}{@{}l}{\textit{Selection and scheduling}}\\
Best-belief quantile & $\epsilon=0.05$ \\
Expansion exponent & $\alpha=2/3$ for $N_t<600$, then $\alpha=1/2$ \\
Concurrent expansions & $\leq 4$ \\
\addlinespace
\multicolumn{2}{@{}l}{\textit{Data splits}}\\
Coder train / validation / test & $28$ / $57$ / $28$ tasks \\
Reviewer train / validation / test & $56$ / $114$ / $56$ rows \\
\addlinespace
\multicolumn{2}{@{}l}{\textit{Limits}}\\
Expansion & $6$\,h \\
Coder episode & $250$ steps, $3$\,h \\
Task container & $2$ CPUs, $8$\,GiB \\
\bottomrule
\end{tabular}
\end{table}

\subsection{Data Splits}
\label{app:splits}

We split every dataset as in \cref{sec:experimental-design}; the held-out protocols are in \cref{app:domain-protocols}. Per dataset:

\noindent\textbf{Coding / \polyglot and \crave.}~~The coder uses $10$ train, $49$ validation, and $166$ held-out \polyglot tasks, disjoint by task. The code reviewer uses \crave's own train, validation, and test partitions: from each, we draw $100$ pull requests with a fixed seed, balanced at $50$ accepted and $50$ rejected, so no row crosses an upstream partition.

\noindent\textbf{Coding / \deepswe.}~~\deepswe~v1.1~\citep{huang2026deepswe} contains $113$ tasks across $91$ repositories and five languages, and publishes no train, validation, or test partition. We derive one from the official trial record, which reports every leaderboard run on every task.

We order tasks by the failures of three reference models from the trial record, \texttt{Claude Fable 5}, \texttt{Claude Opus 5}, and \texttt{GPT-5.6 Sol}, across all five reasoning efforts ($6{,}679$ scored trials). A task's difficulty key is the lexicographic tuple of its per-effort failure rates, ordered from max to low effort, so a task that resists a stronger attempt always ranks as harder. This key separates $106$ of the $113$ tasks, against only $25$ for a key built from the max-effort failure rate alone.

We allocate whole repositories, never individual tasks, so no codebase appears in two splits. Among $400$ seeded greedy allocations refined by equal-size repository swaps, we keep the partition that best matches the three splits on the max-effort failure rate, the pooled failure rate across efforts, and language shares. The result has $28$ train, $57$ validation, and $28$ test tasks. Validation is the largest split because it carries the search signal.

Each task contributes one coder row and two reviewer rows in the same split: an accepted patch, taken from a real passing trial, and a rejected patch, taken from a real failing trial of the same three models. The rejected patch is chosen at the highest reasoning effort that failed ($80$ of $113$ come from max effort) and, within it, as the most nearly correct attempt, with a mean of $91\%$ of verifier tests passing. Both patches are model-written, so the pair differs in correctness alone. This gives $56$, $114$, and $56$ balanced review rows. Neither role sees the verifier tests or the reference solution.

\noindent\textbf{Paper writing / review.}~~We draw $100$ \apres papers for each of train, validation, and test with a fixed seed, balanced at $50$ accepted and $50$ rejected, and disjoint by paper. The writer dataset of \apres titles and abstracts reuses exactly the same papers in each split, so no held-out paper is drafted or reviewed at any point during the search.

\noindent\textbf{IMO grading / proof writing.}~~We draw $300$ \imogradbench human-graded proofs with a fixed seed and split them into $100$ train, $100$ validation, and $100$ test items, stratified by grade. Provers are scored on $20$ held-out IMO problems by the fixed grader panel of \cref{tab:imo-proof-cross-grader}.

The paper-domain mechanism and cost ablations reuse the paper splits, same for the proof domain~(\cref{app:nemotron-ablations,app:hybrid-costs}).

\noindent\textbf{Reward design.}~~Each of the three populations has six rows, one per independent reward-generation and PPO episode~(\cref{app:reward-design}). Train rows are walking episodes that only the meta-agent can request. Validation rows are separate walking episodes with identical task text and simulator settings that score nodes; their composition changes with the curriculum. Held-out rows are six jumping episodes that no search-time task pool or scorer ever references.

\subsection{Held-Out Evaluation Protocols}
\label{app:domain-protocols}

\noindent\textbf{Coding / \polyglot.}~~Generated patches are checked by execution on $166$ held-out tasks~(\cref{tab:raw-counts}).

\noindent\textbf{Coding / \deepswe.}~~The \astra reference is the official trial record filtered to our $28$ held-out tasks, four attempts per task. \Cref{fig:intro-token-efficiency} also shows its full-benchmark rates~(\cref{tab:deepswe-representativeness}).

\begin{table}[tbp]
\centering
\caption{\textbf{\astra pass rates (\%) on our held-out tasks and on the full benchmark.}}
\label{tab:deepswe-representativeness}
\small
\begin{tabular}{l *{5}{S[table-format=2.1,round-mode=none]}}
\toprule
Tasks & {low} & {medium} & {high} & {xhigh} & {max} \\
\midrule
Held-out ($n=28$) & 68.8 & 76.8 & 83.0 & 80.4 & 80.4 \\
Full \deepswe ($n=113$) & 67.0 & 72.8 & 73.2 & 74.1 & 73.2 \\
\bottomrule
\end{tabular}

\vspace{0.6em}
{\itshape Official mini-SWE-agent trials, four attempts per task.}
\end{table}

Each evolved coder runs under the benchmark's container limits and a three-hour deadline, and the official verifier grades its committed patch~(\cref{tab:deepswe-raw}).

\noindent\textbf{Paper writing / review.}~~Reviewers are scored on $100$ held-out \apres papers for accuracy and acceptance rate~(\cref{tab:raw-counts}); each writer produces $100$ papers, judged by the panel of \cref{tab:paper-writer-cross-reviewer}.

\noindent\textbf{IMO grading / proof writing.}~~Graders are scored on $100$ held-out \imogradbench items for exact match and normalized mean absolute error. Each prover writes proofs for $20$ IMO problems, scored by three graders~(\cref{tab:proof-per-grader}); \texttt{Pass@6} and \texttt{Pass@7} count proofs scoring $\geq 6$ and $7$.

\noindent\textbf{Uncertainty.}~~Writer-panel cells~(\cref{tab:paper-writer-cross-reviewer}) report $95\%$ Jeffreys half-widths over $N=100$ papers, and the Mean column pools all $N=400$ paper--reviewer outcomes. Prover-panel cells~(\cref{tab:imo-proof-cross-grader}) pool $N=60$ outcomes ($20$ problems $\times$ $3$ graders) and report one standard error.

\noindent\textbf{Cost-reduction ablation.}~~The \nemotronultra ablation~(\cref{sec:reducing_search_costs}) follows the paper protocol above, with \nemotronultra for search-time task-agent calls and \gptlow (low) for expansion and final \apres scoring. We convert its cost to \gptlow-price-equivalent blended tokens using \$$5$/\$$30$ per M input/output tokens for \gptlow~\citep{openai2026pricing} and \$$0.5$/\$$2.2$ for \nemotronultra~\citep{deepinfra2026pricing}. About $20\%$ of search cost is expansion, charged at \gptlow prices, and $80\%$ is input-heavy task-agent evaluation, charged at the input-price ratio $0.10$; a \nemotronultra total $B$ therefore costs about $0.20B+0.80(0.10B)=0.28B$. The selected reviewer ($B=6.68$M) costs $1.87$M \gptlow-equivalent tokens.

\subsection{Reward-Design Curriculum}
\label{app:reward-design}

\noindent\textbf{Task.}~~We port the Genesis reward-design domain of \hyperagents~\citep{zhang2026hyperagents}. The task agent receives a one-sentence task description and writes a \texttt{compute\_reward} function for a Unitree Go2 quadruped. A fixed host then trains a fresh PPO policy on that reward ($4{,}096$ parallel environments, $101$ iterations) and evaluates checkpoint $100$ for $1{,}000$ steps under the host's own fixed fitness, so the generated reward shapes training but never scores itself. During search the task is to walk forward at a commanded speed in $[0.2,0.8]$\,m/s, scored by command-tracking fitness in $[0,1]$. The held-out task is jumping as high as possible, scored by maximum torso height in meters; a trial counts as a jump when it exceeds the $0.3494$ score. The static arm and the curriculum arm share every setting except the curriculum.

\noindent\textbf{Curriculum.}~~The obstacle is composed of two $0.8$\,m-high slabs meeting at a right angle, initially placed $2$\,m ahead of the robot; the robot starting positions are spread $\pm0.25$\,m sideways, so some robots walk straight into the obstacle/corner and others pass it. The observation carries no exteroception, so the policy must respond to the obstacle from its own state. Stalled forward progress and backing away from the obstacle cost fitness. The validation pool grows by extension:

\begin{itemize}[leftmargin=*,itemsep=0.1em]
\item \textbf{Stage 1} (from evaluation $1$): two rows of the original walking task.
\item \textbf{Stage 2} (from evaluation $20$): adds two rows with $0.5$\,m-wide slabs.
\item \textbf{Stage 3} (from evaluation $50$): adds two rows with $1.0$\,m-wide slabs, which are hard to walk around once the robot is caught in the corner between them.
\end{itemize}

The task objective of walking forward remains on obstacle rows, with the fitness re-weighted partly from speed tracking toward holding heading. The training rows follow the same schedule, and the static arm keeps its six walking rows.

\noindent\textbf{Protocol.}~~We run five searches per arm to $125$ evaluations, where both arms publish comparable numbers of agents ($22.0$ and $21.2$ per run on average). Each run's best-belief agent then writes six fresh jumping rewards; each is trained and evaluated once, and the run's score is the mean of the six. A generated reward that fails to load scores $0$. Every PPO run draws a fresh seed. 

\noindent\textbf{Cost.}~~Both arms use \nemotronultra at high reasoning for the meta-agent and task agent. Per run, the static arm spends $325$ PPO trainings and $48.3$M blended search tokens, and the curriculum arm $291$ trainings and $53.8$M~(\cref{tab:genesis_results}). \hyperagents reports about $600$ trainings and $33.6$M tokens with \texttt{Claude Sonnet 4.5} over about $100$ agents per run, which would be about \$$104$ at list price~\citep{anthropic2026pricing}, compared with \$$24$--\$$27$ for our runs at \nemotronultra list price~\citep{deepinfra2026pricing}, a $3.85$--$4.33\times$ reduction in cost.

\subsection{Initial Agents and Prompts}
\label{app:prompts}

All \gptlow experiments start from one workspace. Every task agent is a single LLM call from the shared template below plus a short per-domain output format, with no role-specific code. The only engineered prompts belong to the controls, the frozen \sakana reviewer and \proofautograder; every checklist, rubric, and acceptance standard reported later was produced by the search~(\cref{app:evaluator-evolution}).

\begin{promptbox}{Shared task-agent template (all roles)}
\begin{PromptVerb}
You are an agent.

Task input:
```
{inputs}
```

{output_format}
\end{PromptVerb}
\end{promptbox}

The paper reviewer's format is shown below; the other roles differ only in the schema they request. The coder also starts with bash and an editor, at most $16$ tool calls per task.

\begin{promptbox}{Initial output format: paper reviewer (also scores writer outputs and the adversarial pool)}
\begin{PromptVerb}
You are reviewing the paper above for a top ML venue. Read it carefully
and decide whether to Accept or Reject it.

Respond with one JSON block:
{
  "Summary": "...", "Strengths": [...], "Weaknesses": [...],
  "Decision": "Accept" | "Reject"
}
For "Decision", use only "Accept" or "Reject".
\end{PromptVerb}
\end{promptbox}

The meta-agent is instructed to improve any part of the codebase, including its own loop, prompts, and tools. It reads each ancestor's train-split predictions, meta-agent transcript, and code diff, but it is never shown any validation outcomes during the search.

Runtime constants are in \cref{tab:settings-editable}.

\noindent\textbf{\deepswe workspace.}~~Both \deepswe searches start from one mini-SWE-agent harness~\citep{yang2025miniswe}: the coder works in the repository through a single bash tool, for at most $250$ steps, and submits the patch it commits. The \ourmethod code reviewer reads the task statement and the diff, never the verifier tests or reference solution, and co-evolves with the coder. The meta-agent may edit every workspace file, including its own loop, but not the fixed scoring harness.

\begin{promptbox}{\deepswe initial code-reviewer instruction}
\begin{PromptVerb}
Review the diff above and assess whether it solves the problem the
statement sets out.

Work from the statement. It is the specification, and the only one: take
the behaviour it asks for -- the inputs, options, calls and cases it
names -- and for each, find the code in the diff that delivers it and
judge whether that code is correct. ...

Accept when the diff correctly tackles the problem and provides the
solution the statement asks for. Reject when some behaviour it asks for
has no code behind it, or has code you can show is wrong. ...

Reply with one JSON block and nothing after it:
{ "Reasoning": "...", "Decision": "Accept" | "Reject" }
\end{PromptVerb}
\end{promptbox}

The \deepswe meta-agent receives the prompt above, adapted to mini-SWE-agent, then this:

\begin{promptbox}{\deepswe initial meta-agent domain instruction (opening, \ourmethod)}
\begin{PromptVerb}
This expansion has a lifetime allowance of 16 individual measured training
results across every batch, code version, turn and resume. ...

A solver reads one real software-engineering task and emits a diff. The
initial reviewer reads the complete diff and problem statement in one
`chat` call and answers accept or reject. Its Python code also receives
inputs["repository"], the path to the complete task repository at its
pinned base revision. The patch has not been applied there. ... You may
evolve the reviewer to read files before its model call, add tools, or
use multiple turns. File selection, paths, searches and prompt
construction are your editable code; the host imposes no source-file
selection.
\end{PromptVerb}
\end{promptbox}

\section{Results and Ablations}
\label{app:extended-results}

\subsection{Held-Out Results}
\label{app:raw-counts}

\begin{table}[tbp]
\centering
\caption{\textbf{Held-out reviewer and coder results} behind \cref{sec:exp-polyglot,sec:exp-writer}.}
\label{tab:raw-counts}
\small
\begin{tabular}{l S[table-format=3.0,round-mode=none] S[table-format=2.0,round-mode=none]}
\toprule
Reviewer & {Accuracy (\%)} & {Accept. (\%)} \\
\midrule
\hgmh specialist (node 51) & 88 & 45 \\
\ourmethod specialist (node 49) & 84 & 40 \\
\ourmethod generalist, pre-repl.\ (node 78) & 84 & 42 \\
\ourmethod adversarial, post-repl.\ (node 31) & 80 & 32 \\
\dgmh (published) & 73 & 25 \\
\sakana prompt & 63 & 13 \\
\bottomrule
\end{tabular}

\vspace{0.6em}
{\itshape Reviewers: held-out \apres papers.}

\vspace{0.9em}
\begin{tabular}{l S[table-format=2.1,round-mode=none]}
\toprule
Best-belief agent & {Pass rate (\%)} \\
\midrule
\ourmethod specialist & 71.7 \\
\ourmethod generalist & 71.7 \\
\hgmh & 69.9 \\
\bottomrule
\end{tabular}

\vspace{0.6em}
{\itshape Coders: held-out \polyglot tasks.}
\end{table}

\begin{table}[tbp]
\centering
\caption{\textbf{\deepswe held-out pass rates (\%)} behind \cref{fig:intro-token-efficiency}.}
\label{tab:deepswe-raw}
\small
\begin{tabular}{l *{3}{S[table-format=2.1,round-mode=none]}}
\toprule
Coder & {low} & {medium} & {high} \\
\midrule
\bfseries \ourmethod coder & \bfseries 82.1 & \bfseries 85.7 & \bfseries 85.7 \\
\hgmh coder & 75.0 & 78.6 & 82.1 \\
\astra (mini-SWE-agent) & 68.8 & 76.8 & 83.0 \\
\bottomrule
\end{tabular}

\vspace{0.6em}
{\itshape \astra is the mean of four official attempts per task.}
\end{table}

\Cref{tab:proof-per-grader} splits the prover panel by grader. The \ourmethod specialist's higher mean over the \imocode prover comes from more near-complete ($6/7$) proofs: on \texttt{Pass@7} it trails \imocode but beats \hgmh. Excluding the \ourmethod grader, which co-evolved with the specialist, the two independent graders give it $4.03$, against $4.08$ for \imocode and $3.60$ for \hgmh.

\begin{widetable}[tbp]
\centering
\small
\setlength{\tabcolsep}{3.2pt}
\caption{\textbf{Per-grader proof scores behind \cref{tab:imo-proof-cross-grader}.}}
\label{tab:proof-per-grader}
\begin{tabular}{l *{9}{S[table-format=2.1,round-mode=none]}}
\toprule
& \multicolumn{3}{c}{\proofautograder} & \multicolumn{3}{c}{\hgmh grader (node 45)} & \multicolumn{3}{c}{\ourmethod grader (node 22)} \\
\cmidrule(lr){2-4}\cmidrule(lr){5-7}\cmidrule(lr){8-10}
Prover & {Mean} & {P@6} & {P@7} & {Mean} & {P@6} & {P@7} & {Mean} & {P@6} & {P@7} \\
\midrule
Static \imocode prover & 4.05 & 55.0 & 55.0 & 4.10 & 55.0 & 55.0 & 4.05 & 55.0 & 55.0 \\
\hgmh prover & 3.55 & 50.0 & 45.0 & 3.65 & 50.0 & 45.0 & 4.00 & 55.0 & 45.0 \\
\ourmethod prover (generalist) & 3.55 & 50.0 & 45.0 & 3.70 & 50.0 & 45.0 & 3.95 & 55.0 & 45.0 \\
\bfseries \ourmethod prover (specialist) & \bfseries 3.85 & \bfseries 55.0 & \bfseries 40.0 & \bfseries 4.20 & \bfseries 60.0 & \bfseries 45.0 & \bfseries 4.95 & \bfseries 70.0 & \bfseries 60.0 \\
\bottomrule
\end{tabular}
\end{widetable}
\FloatBarrier

\subsection{Mechanism Ablations}
\label{app:nemotron-ablations}

We ablate evaluator replacement, the adversarial pool, and selective erasure in the paper domain with \nemotron. In the proof domain, we compare replacement and pool interventions with \nemotronultra against a coupled \hgm control. Setup details are in \cref{app:fm,app:splits,app:baselines-selection}.

\subsubsection{Paper Domain}
\begin{widetable}[tbp]
\centering
\caption{\textbf{Mechanism ablations on the evaluator-independent anchor.} All runs use a $12{,}288$-evaluation budget. Reviewer anchor accuracy is comparable across runs, while writer acceptance ($\dagger$), scored by each run's own reviewer, is not. Cells report mean $\pm$ 95\% Jeffreys interval.}
\label{tab:nemotron-runs}
\small
\setlength{\tabcolsep}{4.5pt}
\begin{tabular}{l c c c c c}
\toprule
Run & Repl. & Pool & Erase & Writer val.\ acc.\ (\%)$^{\dagger}$ & Reviewer anchor acc.\ (\%) \\
\midrule
Full mechanism & on & on & on & $90.1$\,{\color{black!55}\scriptsize$\pm 3.5$} & $95.0$\,{\color{black!55}\scriptsize$\pm 10.3$} \\
Replacement only & on & off & on & $89.5$\,{\color{black!55}\scriptsize$\pm 13.7$} & $91.2$\,{\color{black!55}\scriptsize$\pm 7.4$} \\
No erasure & on & on & off & $90.0$\,{\color{black!55}\scriptsize$\pm 10.7$} & $92.9$\,{\color{black!55}\scriptsize$\pm 10.8$} \\
\addlinespace
Fixed reviewer & off & off & {--} & $100.0$\,{\color{black!55}\scriptsize$\pm 1.7$} & $78.2$\,{\color{black!55}\scriptsize$\pm 4.4$} \\
\bottomrule
\end{tabular}
\end{widetable}

\Cref{tab:nemotron-runs} compares the full mechanism with runs that drop the pool, erasure, or replacement. 

\noindent\textbf{The fixed reviewer saturates.}~~Without replacement, the writer ends with all $75$ validation papers accepted, far above the human acceptance rate of the \apres data.

\noindent\textbf{Replacement escapes that plateau.}~~With replacement but no pool, the reviewer reaches $52/57$ on the anchor; the writer re-bases after each replacement instead of saturating one judge. The pool then decides \emph{which} reviewer wins~(\cref{sec:exp-adversarial}), which matters most for reviewing generated papers.

\noindent\textbf{Without erasure the utility becomes outdated.}~~Rows scored by retired reviewers stay in the utility, and replacements stop as outdated evidence accumulates, making the search revisit the same nodes.

\subsubsection{Proof Domain}
\label{app:proof-mechanism-ablations}

\begin{widetable}[tbp]
\centering
\caption{\textbf{Proof-writing mechanism ablations on the full held-out set.}
Completed proofs are scored by a common three-grader panel.
Cells report the mean $0$--$7$ score with $95\%$ problem-cluster bootstrap intervals over all held-out problems in the campaign.}
\label{tab:imo-mechanism-ablations}
\small
\begin{tabular}{l c c}
\toprule
Condition & Prover-specialist score $\uparrow$ \\
\midrule
Full mechanism & $3.60$\,{\color{black!55}\scriptsize$[2.55, 4.67]$} \\
Pool off & $3.14$\,{\color{black!55}\scriptsize$[2.13, 4.20]$} \\
Transitions off & $2.64$\,{\color{black!55}\scriptsize$[1.74, 3.55]$} \\
Coupled \hgm & $3.05$\,{\color{black!55}\scriptsize$[2.11, 4.04]$} \\
\bottomrule
\end{tabular}
\end{widetable}

\begin{widetable}[tbp]
\centering
\caption{\textbf{Proof-grading mechanism ablations.}
All four graders are evaluated on the same $548$ human-labeled
items across $20$ problems. Brackets contain $95\%$ problem-cluster
bootstrap intervals.}
\label{tab:imo-grader-mechanism-ablations}
\small
\begin{tabular}{l c c}
\toprule
Condition & Exact agreement (\%) $\uparrow$ & MAE$/7$ $\downarrow$ \\
\midrule
Full mechanism & $67.3$\,{\color{black!55}\scriptsize$[57.6, 76.9]$} & $0.122$\,{\color{black!55}\scriptsize$[0.088, 0.160]$} \\
Pool off & $58.8$\,{\color{black!55}\scriptsize$[47.6, 69.8]$} & $0.169$\,{\color{black!55}\scriptsize$[0.122, 0.217]$} \\
Transitions off & $56.0$\,{\color{black!55}\scriptsize$[44.7, 67.2]$} & $0.197$\,{\color{black!55}\scriptsize$[0.146, 0.247]$} \\
Coupled \hgm & $65.7$\,{\color{black!55}\scriptsize$[55.9, 75.1]$} & $0.141$\,{\color{black!55}\scriptsize$[0.098, 0.185]$} \\
\bottomrule
\end{tabular}
\end{widetable}

\Cref{tab:imo-mechanism-ablations,tab:imo-grader-mechanism-ablations} compare the full mechanism against runs without the pool, without transitions, and a coupled \hgm control. Prover specialists and grader specialists are selected without held-out scores.

\noindent\textbf{The full mechanism finds the best grader.}~~It reaches $67.3\%$ agreement with human grades, against $65.7\%$ for coupled \hgm, $58.8\%$ without the pool, and $56.0\%$ without transitions, and it also has the lowest normalized grading error of the four conditions compared.

\noindent\textbf{The full mechanism finds the best prover.}~~The full-mechanism specialist scores $3.60/7$, against $3.05$ for coupled \hgm, $3.14$ without the pool, and $2.64$ without transitions.

\subsection{Where Cost and Progress Come From}
\label{app:analyses}

\subsubsection{Cost decomposition}
\label{app:analyses:budget}

Every model call is an expansion, a train-time evaluation, or a validation evaluation. Validation dominates: across the three \gptlow runs it carries $65$--$69\%$ of blended tokens, expansion $18$--$23\%$, and train-time evaluation $12$--$14\%$ at three train samples per node~(\cref{fig:analysis-cost}).

\begin{figure}[tbp]
\centering
\includegraphics[width=\linewidth]{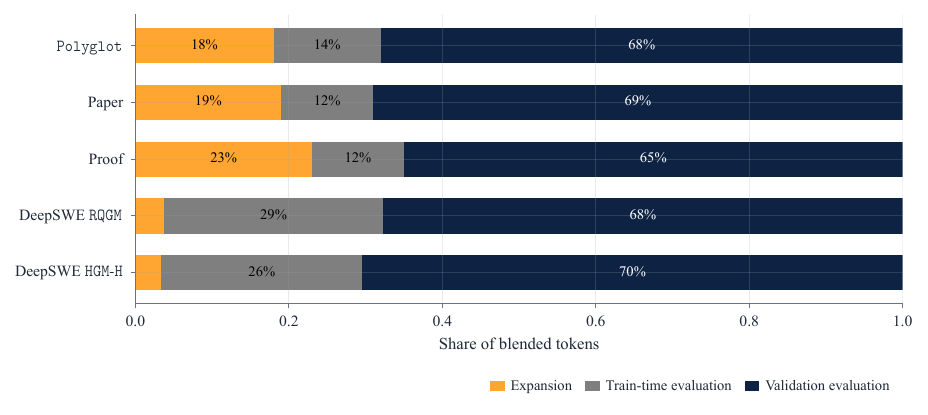}
\caption{\textbf{Blended-token cost of the five headline runs} by expansion, train-time evaluation, and validation evaluation, for the three \gptlow runs and the two \deepswe runs.}
\label{fig:analysis-cost}
\end{figure}

The \deepswe runs follow the same pattern with a smaller expansion share, since each evaluation is a long agentic episode: validation carries $68\%$ and $70\%$ of the \ourmethod and \hgmh runs, train-time evaluation $29\%$ and $26\%$, and expansion $4\%$ and $3\%$~(\cref{tab:deepswe-cost}). Before each winner was generated, most tokens went to the train-time evaluations the meta-agents request while editing.

\begin{table}[tbp]
\centering
\caption{\textbf{Blended-token decomposition of the \deepswe searches} (millions). \emph{To winner} rows count tokens until each best coder was generated; \emph{full run} rows count the full search.}
\label{tab:deepswe-cost}
\small
\begin{tabular}{l *{4}{S[table-format=4.1,round-mode=none,group-separator={,}]}}
\toprule
Run & {Expansion} & {Train} & {Validation} & {Total} \\
\midrule
\ourmethod, to winner & 55.5 & 446.1 & 85.6 & 587.2 \\
\ourmethod, full run & 107.3 & 821.4 & 1948.2 & 2876.8 \\
\addlinespace
\hgmh, to winner & 44.0 & 303.6 & 81.0 & 428.6 \\
\hgmh, full run & 85.1 & 664.4 & 1783.2 & 2532.7 \\
\bottomrule
\end{tabular}
\end{table}

\subsubsection{Patch surfaces}
\label{app:analyses:patches}

We classify each accepted patch by the code it modifies. Shared surfaces, task-agent code and infrastructure used by more than one role, carry $59$--$90\%$ of accepted patches in every \gptlow run and dominate in \polyglot~(\cref{fig:analysis-transfer}); the \nemotron run instead targets the meta-agent's instruction module in $42\%$ of edges. On \deepswe, $83\%$ of \ourmethod code patches do~(\cref{app:lineages}).

\begin{figure}[tbp]
\centering
\includegraphics[width=\linewidth]{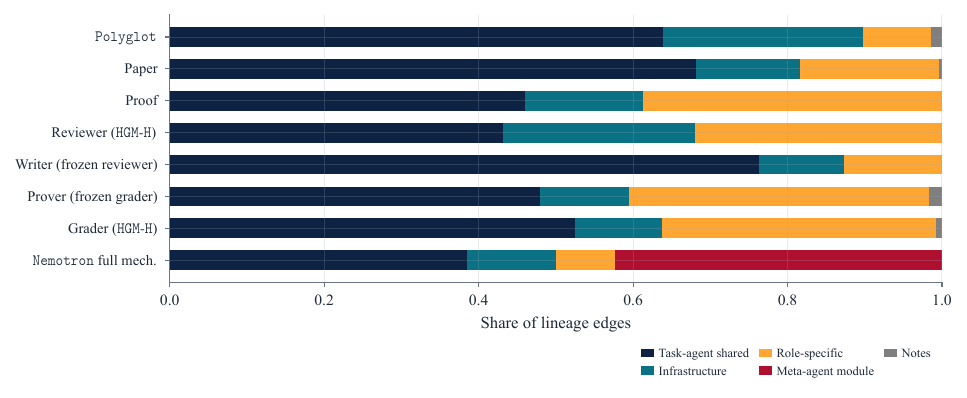}
\caption{\textbf{Patch surfaces by run}, classifying each accepted lineage edge by the code surface it modifies (task-agent shared, infrastructure, role-specific, meta-agent module, or notes).}
\label{fig:analysis-transfer}
\end{figure}

\subsection{Reducing search costs via hybrid-model approaches}
\label{app:hybrid-costs}
\label{sec:reducing_search_costs}

Expansion is only about $20\%$ of search cost~(\cref{fig:analysis-cost}). We therefore keep the \gptlow (low) meta-agent and route search-time task-agent calls through \nemotronultra~\citep{nvidia2026nemotron3ultraopen}.

\begin{figure}[tbp]
  \centering
  \includegraphics[width=0.85\linewidth]{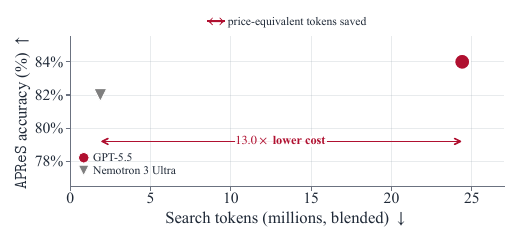}
  \caption{\textbf{\nemotronultra task-agent calls approach the \gptlow-only run at $\approx13.0\times$ lower price-equivalent search cost.} Held-out \apres accuracy is scored with \gptlow in both runs; costs are \gptlow-price-equivalent blended tokens~(\cref{app:domain-protocols}).}
  \label{fig:nemotron-cost-shift}
\end{figure}

This cuts price-equivalent search cost by about $13\times$ while nearly matching the \gptlow-only run's held-out accuracy on the $100$ \apres test papers~(\cref{fig:nemotron-cost-shift}).

\section{What the Co-Evolution Discovered}
\label{app:evaluator-evolution}

Each evaluator is a prompt entry in the shared workspace, so its history can be read off its lineage. The search's most common move, in every role, turns the one-line initial instruction~(\cref{app:prompts}) into a specific rubric, and we quote the rules it produced below.

\subsection{What the Evolved Evaluators Changed}
\label{app:evaluator-evolution:reviewer}

The adversarial reviewer's lineage rewrote the initial instruction into a conjunctive acceptance standard; the displaced generalist carries a two-sided calibration.

\begin{promptbox}{Evolved rule: the adversarial reviewer's conjunctive standard (gen\,4)}
\begin{PromptVerb}
... Use a realistic selective-conference standard: accept only if the
paper has a clear nontrivial contribution, technically sound
methodology/theory or experiments, adequate positioning against prior
work, and claims supported by evidence; reject papers with missing
validation, unclear novelty, serious technical gaps, or overclaiming.
Do not be swayed by fluent writing alone.
\end{PromptVerb}
\end{promptbox}

\begin{promptbox}{Evolved rule: the displaced generalist's two-sided calibration}
\begin{PromptVerb}
... Calibrate the decision holistically rather than requiring
perfection: a paper with a coherent contribution and adequate evidence
should be accepted despite presentation flaws, while impressive writing
without substantiated experiments or technical substance should be
rejected.
\end{PromptVerb}
\end{promptbox}

Both wordings appeared early; the adversarial pool decided which lineage was selected, favoring the one that rejects more \AI-generated papers without losing anchor accuracy.

\subsection{The Co-Evolved Grader}
\label{app:evaluator-evolution:grader}

The grader took its final form in two patches. The initial instruction assigns one of four labels with no rubric. The first patch adds a persona, a two-step milestone procedure, a points mapping, and a strictness rule, and introduces the prover prompt; a later patch aligns it with human grades.

\begin{promptbox}{Evolved rule: the grader's first rewrite (gen\,2)}
\begin{PromptVerb}
You are an expert, conservative IMO proof grader. ... First identify
the key required steps from the guidelines; then check whether the
student's text actually proves them rigorously, ignoring unsupported
claims, handwaving, and statements merely resembling the reference.

Use this mapping unless the problem-specific guidelines say otherwise:
- correct: a complete rigorous solution (7/7).
- almost: essentially complete with only a minor local gap or
  arithmetic slip (6/7); do NOT use for missing major lemmas.
- partial: substantial guideline-listed progress, but not a
  near-complete proof (1/7).
- incorrect: no substantial credited progress or a fundamentally
  flawed argument (0/7).

Be especially strict with advanced geometry/combinatorics solutions
that invoke large unproved lemmas ...
\end{PromptVerb}
\end{promptbox}

The second patch (gen\,22) aligns the grader with human grades: it credits proofs that differ from the reference solution or cite standard results tersely, and fails proofs with fatal gaps.

These two patches explain its behavior in the panel of \cref{tab:imo-proof-cross-grader}.

\subsection{The \polyglot Code Reviewer}
\label{app:evaluator-evolution:critic}

The \polyglot code-review slot was replaced three times, starting from a pass/fail verdict on general code quality. Each incumbent narrows what counts as a failure to defects visible in the diff. The third branches from an earlier ancestor than the second, so it does not inherit the second's rules.

\begin{promptbox}{Evolved rules: the three promoted code reviewers}
\begin{PromptVerb}
Node 19 (gen 5): You are a pragmatic senior maintainer reviewing a
proposed patch. ...
- PASS when the patch plausibly fixes/adds the intended behavior, is
  localized, and has no clear correctness, compilation, security, or
  data-loss defect. ...
- FAIL only for a concrete blocker: obvious wrong behavior, broken API
  usage, compile/type errors visible in the diff, missing essential
  cases, flaky/meaningless tests, ...
- ... If evidence is balanced, prefer pass.

Node 34 (gen 34, appends): - Review in this order: (1) infer the
intended user-visible behavior ..., (2) check whether production code
implements it, (3) look for definite regressions or syntax/type
errors, (4) only then consider style, test coverage, and
maintainability. Style concerns without a concrete correctness risk
should not cause fail. ... FAIL only when you can point to a specific
scenario that remains broken ...

Node 86 (gen 59, extends the blocker rule): ... Do not invent blockers
from missing project context; the diff alone must show the defect.
- PASS broad but coherent mechanical migrations, CI/build-artifact
  pipeline updates, ... and test-only fixes when they are internally
  consistent ...
- For security/credential-looking files, fail only for an actual
  exposed secret or dangerous permission; ...
\end{PromptVerb}
\end{promptbox}

\subsection{The \deepswe Code Reviewer}
\label{app:evaluator-evolution:deepswe}

The \deepswe code-review slot was replaced $11$ times across $8$ distinct reviewers, three times before the winner was generated, at $47$M, $160$M, and $309$M search tokens~(\cref{fig:intro-token-efficiency}). The seated reviewer's ground-truth score rose at each of those three replacements, from $0.03$ to $0.12$, $0.19$, and $0.22$, and reached $0.48$ by the last. 

\noindent\textbf{What the seated reviewers changed.}~~The first replacement kept the root's single-call design, so its gain came from evidence. Later reviewers read the repository before deciding. The second writes an acceptance checklist from the task statement alone, before seeing the patch, then inspects source in up to three rounds. The last and strongest applies the patch to a scratch copy with Git.

\noindent\textbf{Does anchor promotion carry over to the search?}~~Promotion only checks a reviewer on the anchor, so we ask whether the promoted reviewer also judges the coder patches the search produces more accurately. Every search-time verdict is joined to the test result of the patch it judged~(\cref{tab:deepswe-reviewer-fidelity}). Every reviewer seated before the last accepts nearly every patch, failing or not, so it gives the search almost no signal. The final reviewer, with the highest anchor score, is the only one that separates failing from passing patches. On the $337$ patches it judged together with an earlier reviewer, it accepts $19\%$ of failing patches against $85\%$ for its predecessors ($95\%$ task-bootstrap interval of the drop $[0.12, 0.80]$), raising balanced accuracy from $0.52$ to $0.64$. It is also harsher on correct work, accepting $47\%$ of passing patches against $89\%$. 

\begin{table}[tbp]
\centering
\caption{\textbf{Only the reviewer with the highest anchor score separates failing from passing \deepswe patches during search.} Each row is a distinct seated reviewer in seating order, excluding the root; \textbf{Anchor} is its ground-truth score at the end of its last tenure, and the acceptance rates are over every coder patch it judged during search, split by the patch's test result.}
\label{tab:deepswe-reviewer-fidelity}
\small
\begin{tabular}{lcrrcc}
\toprule
\textbf{Reviewer} & \textbf{Anchor} & \textbf{Patches} & \textbf{Failing} & \textbf{Accepts passing} & \textbf{Accepts failing} \\
\midrule
\texttt{beec9b7} & $0.22$ & $230$ & $26$ & $95.6\%$ & $100.0\%$ \\
\texttt{5f52214} & $0.19$ & $39$ & $1$ & $100.0\%$ & $100.0\%$ \\
\texttt{09e57aa} & $0.22$ & $150$ & $14$ & $94.1\%$ & $92.9\%$ \\
\texttt{03333c6} & $0.22$ & $262$ & $37$ & $91.6\%$ & $83.8\%$ \\
\texttt{8fcb4fe} & $0.22$ & $232$ & $33$ & $88.9\%$ & $81.8\%$ \\
\texttt{0051587} & $0.36$ & $201$ & $25$ & $92.0\%$ & $92.0\%$ \\
\texttt{54ce092} (final) & \textbf{0.48} & $1181$ & $298$ & $35.0\%$ & \textbf{20.8\%} \\
\bottomrule
\end{tabular}
\end{table}

\subsection{Winning Lineages}
\label{app:lineages}

\begin{widefigure}[tbp]
    \centering
    \includegraphics[width=\linewidth]{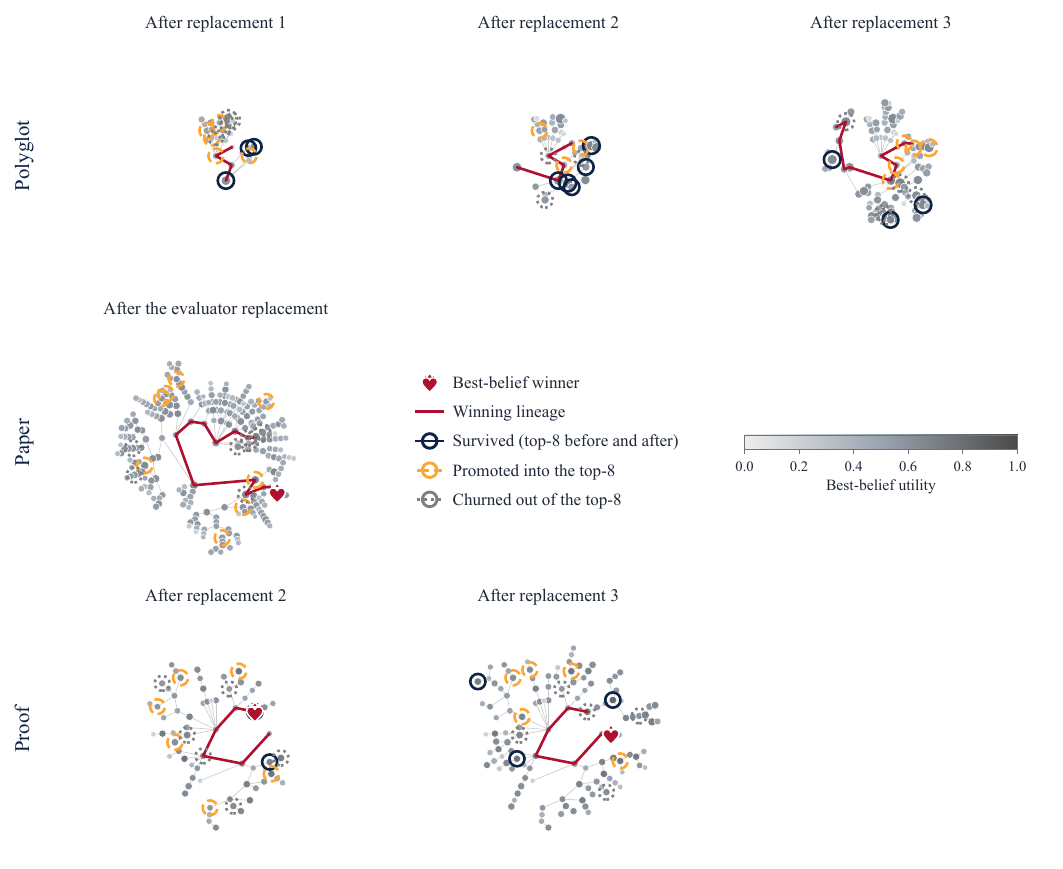}
    \caption{\textbf{What each evaluator replacement keeps versus what it re-ranks, in all three headline domains.} Rows are domains (top to bottom: \polyglot, paper, proof); columns are that run's shown evaluator replacements, each drawing the archive as it stood at that replacement (radial layout, node color is best-belief utility and size is evidence count). Among the top-$8$ leaderboard nodes at each replacement, navy rings survived the re-ranking, amber rings are lineages the new criterion promotes into the top-$8$, and gray rings are churned out. Replacements with fewer than $20$ ranked nodes are omitted, as in \cref{fig:transition-rerank}, so the proof run's first replacement is not shown.}
    \label{fig:lineage-trees}
\end{widefigure}

Each archive node stores its parent, its cumulative diff, and the meta-agent transcript that produced it. \Cref{fig:lineage-trees} draws the three \gptlow archives, \cref{tab:chains-polyglot,tab:chains-paper,tab:chains-proof} list every step, and quotes are from the meta-agents' own summaries. 

\begin{widetable}[tbp]
\centering
\caption{\textbf{\polyglot winning chains; both final best-belief agents reach $71.7\%$ held out.}}
\label{tab:chains-polyglot}
\small
\begin{tabular}{@{}r >{\raggedright\arraybackslash}p{0.84\linewidth}@{}}
\toprule
{Node} & Modification \\
\midrule
\multicolumn{2}{@{}l}{\textit{specialist}}\\
4 & coder tool budget $16{\to}40$; review calibration; catalog-path fix \\
35 & coder workflow: one tool call at a time, prefer exact algorithms \\
51 & review balance against blindly passing bad patches \\
104 & parse retry; label normalization; concrete-blocker review rule \\
\midrule
\multicolumn{2}{@{}l}{\textit{generalist}}\\
5 & parse retry; tool budget $16{\to}24$; concrete-blocker review \\
19 & untracked-artifact and symlink diff hygiene \\
33 & the same hygiene applied to a second code path \\
42 & one tool call per turn; preserve existing public APIs \\
72 & read problem text and tests first; wider artifact filter \\
78 & final-output reminder; conservative label normalization \\
93 & dedicated artifact-review format; pass plausible localized fixes \\
100 & JSON-serialized tool results (transcript-corruption fix) \\
\bottomrule
\end{tabular}
\end{widetable}

\begin{widetable}[tbp]
\centering
\caption{\textbf{Paper-domain winning chains.} The writer chain passes through the reported reviewer (49) and generalist (78); the adversarial reviewer (31) descends separately.}
\label{tab:chains-paper}
\small
\begin{tabular}{@{}r >{\raggedright\arraybackslash}p{0.84\linewidth}@{}}
\toprule
{Node} & Modification \\
\midrule
\multicolumn{2}{@{}l}{\textit{writer}}\\
7 & stronger authoring prompt; runtime prompt nudges \\
11 & parse-repair retry; $8{,}192$-token writer output cap \\
17 & present raw paper text instead of the input dictionary \\
37 & remove global LLM monkeypatch \\
49 & writer returns raw text, not JSON-wrapped continuations \\
78 & two-sided review calibration; no acceptance-seeking prose \\
81 & format hygiene in the client; revert import hook that could hang \\
89 & continue from the exact stopping point; longer word target \\
150 & liberal unwrapping of writer output (recovers wrapped papers) \\
154 & \textbf{area-chair review standard}; bounded-evidence writing rules \\
\midrule
\multicolumn{2}{@{}l}{\textit{adversarial}}\\
4 & conjunctive selective-conference review standard \\
25 & writer-side review treated as a viability gate \\
31 & self-contained revised paper; $4{,}096$-token output cap \\
\bottomrule
\end{tabular}
\end{widetable}

\begin{widetable}[tbp]
\centering
\caption{\textbf{Proof-domain winning chains.} The grader's alignment step at node 22 is the largest single-step gain in any winning chain of the three runs.}
\label{tab:chains-proof}
\small
\begin{tabular}{@{}r >{\raggedright\arraybackslash}p{0.84\linewidth}@{}}
\toprule
{Node} & Modification \\
\midrule
\multicolumn{2}{@{}l}{\textit{grader}}\\
2 & conservative rubric-aware grading; first prover prompt \\
22 & alignment with human grades on terse or compressed proofs \\
\midrule
\multicolumn{2}{@{}l}{\textit{prover}}\\
3 & parse retry; lemma-justification and case-coverage prompt \\
11 & partial-versus-incorrect distinction; hidden-assumption checks \\
19 & second-pass check with an adversarial checker \\
66 & preserve sound drafts (below) \\
\bottomrule
\end{tabular}
\end{widetable}

The most consequential code edit outside evaluator prompts is node 66's change to the prover's inherited revision pass~(\cref{lst:prover-retraction}), which preserves sound drafts instead of rewriting them.

\begin{lstlisting}[style=rsdiff, caption={Node 66's change (proof run): the inherited revision pass is off by default.}, label={lst:prover-retraction}]
diff --git a/task_agent.py b/task_agent.py
@@ class TaskAgent(AgentSystem):
-    IMO_PROOF_REVISION_ENABLED = os.environ.get("HYPERAGENTS_IMO_PROOF_REVISION", "1") != "0"
+    IMO_PROOF_REVISION_ENABLED = os.environ.get("HYPERAGENTS_IMO_PROOF_REVISION", "0") == "1"
@@ Draft proof:
-Find any hidden gaps ... If the draft is sound, keep it but make it cleaner.
+Preserve the draft if it is essentially sound; do not replace a detailed proof
+with a shorter answer that relies on an unproved obscure named theorem.
\end{lstlisting}

\begin{widefigure}[tbp]
\centering
\includegraphics[width=\linewidth]{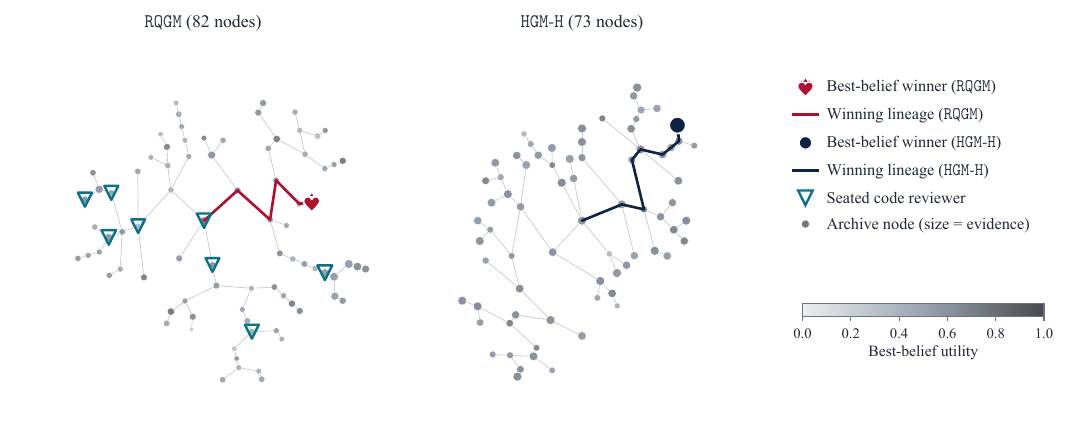}
\caption{\textbf{The \ourmethod and \hgmh \deepswe archives.} Radial layout; node color is the final best-belief coder utility and size is evidence count. The winning lineages end in each run's best coder. Triangles mark nodes seated as the code reviewer.}
\label{fig:tree-deepswe}
\end{widefigure}

On \deepswe, the \ourmethod archive holds $82$ nodes and the \hgmh archive $73$~(\cref{fig:tree-deepswe}); \cref{tab:chains-deepswe} lists both winning chains, quoting the meta-agents' commit messages. Of the $52$ accepted \ourmethod patches that change code, $83\%$ modify shared code or both roles.

\begin{widetable}[tbp]
\centering
\caption{\textbf{\deepswe winning chains.} The \ourmethod coder reaches $82.1\%$ held out at low effort, the \hgmh coder $75.0\%$. Nodes marked with a dash published no code change: the host keeps each expansion's best-measured submission, which for these two was the inherited code.}
\label{tab:chains-deepswe}
\small
\begin{tabular}{@{}r >{\raggedright\arraybackslash}p{0.84\linewidth}@{}}
\toprule
{Node} & Modification \\
\midrule
\multicolumn{2}{@{}l}{\textit{\ourmethod coder}}\\
\texttt{e28cf32f} & repository-grounded two-call reviewer (research, then audit); contract-checklist coder preamble \\
\texttt{1bf78e50} & -- (concise-revision experiment not published) \\
\texttt{f076c908} & -- (scenario audit and source-diverse retrieval not published) \\
\texttt{337755f2} & review treated as a fallible hypothesis in revision; final judge no longer shown the investigator's hypotheses \\
\texttt{904d5a2a} & coder adds new regression tests while keeping existing tests and callers unchanged \\
\midrule
\multicolumn{2}{@{}l}{\textit{\hgmh coder}}\\
\texttt{5d7a3a5c} & repair treated as a hypothesis \\
\texttt{a79459ad} & robust JSON parsing \\
\texttt{e5acee5a} & coder may add new regression tests; existing tests read-only \\
\texttt{73527b78} & execution-evidence log and verification audit added, both disabled \\
\texttt{90c047d6} & task-only contract probes added, disabled after failed replication \\
\texttt{d229327d} & evidence-first issue adjudication \\
\texttt{34bcf8a1} & one-paragraph contract map in the coder system prompt \\
\texttt{b999b5e7} & compact coder instance template \\
\bottomrule
\end{tabular}
\end{widetable}

\noindent\textbf{What the \ourmethod chain changed.}~~Every code change on the chain targets the reviewer or the coder's contract. The first expansion split the reviewer into two calls: a research call reads the problem and diff and requests up to eight reads of the base repository, and a judge call must back every rejection with a concrete failing call. The same patch replaced the coder preamble with a contract workflow.

The next two expansions kept their parent's code. The fourth stopped showing the judge the investigator's unverified hypotheses. The winner adds regression tests and leaves existing ones unchanged.

\begin{promptbox}{\deepswe evolved reviewer audit (node \texttt{e28cf32f}, first line from \texttt{337755f2})}
\begin{PromptVerb}
Independently trace the requested behavior through the proposed code
before deciding. For each proposed rejection give a concrete execution
or API call, the required result, the actual result, and the code path
responsible. The repository extracts show the BASE revision, so mentally
apply the complete diff; removed base code is not a defect in the
proposed result. Missing extracts are not missing implementation. ...
\end{PromptVerb}
\end{promptbox}

\noindent\textbf{What the \hgmh chain changed.}~~Its third node reached the regression-test permission the \ourmethod winner adopted. The \ourmethod reviewer was replaced three times under its anchor before the winner.

\section{Theory}
\label{app:proofs}

\subsection{Setting: the \texorpdfstring{\hgm}{HGM} Search Algorithm}
\label{app:prelim}
\label{sec:prelim:tree}
\label{sec:prelim:cmp}

We use the tree search of \cref{sec:method:formal}, with $U(a)$ the expected success rate under the evaluation protocol.

\begin{theorem}[\citet{wang2025huxley}]
\label{thm:cmp}
Under \cref{assumption:hgm}, access to a \cmp oracle is sufficient to implement the G\"odel Machine within the given search budget.
\end{theorem}

\Theourmethod keeps \hgm's Thompson sampling over clade counts and its UCB-Air gate~\citep{NIPS200849ae49a2}, with the role--task balancing and erasure described above. The results below concern the stream of binary outcomes each node receives; what a node contains does not enter.

\subsubsection{Operating Conditions}
\label{sec:prelim:assumptions}

\begin{assumption}[Fixed-criterion search]
\label{assumption:hgm}
Within one fixed-criterion epoch:
\begin{enumerate}[leftmargin=*,itemsep=0.2em,label=\textbf{C\arabic*.},ref=C\arabic*]
    \item \label{cond:final} Only the agent selected after the budget is spent is scored; intermediate outcomes guide search but are not rewards.
    \item \label{cond:stationary} Repeated evaluations of a fixed node--task pair are binary outcomes with a time-invariant success probability.
    \item \label{cond:proofs} Any oracle or proof object in the ideal \hgm theorem lies outside the empirical budget.
    \item \label{cond:budget} Expansions and evaluations are charged against the externally fixed budget.
\end{enumerate}
\end{assumption}

\Cref{cond:stationary} is the one \ourmethod must work to keep. \hgm relies on a fixed evaluator for it; \ourmethod freezes learned evaluators within an epoch~(\cref{sec:method:generative}), changes them only at checkpoints, and erases records that depended on a displaced evaluator~(\cref{sec:method:utility}); all results below assume \cref{assumption:hgm}.

\subsection{One Rule for Replacement and Selection}
\label{app:controlled_utility_details}

Evaluators are frozen within an epoch. At a checkpoint, each slot promotes the candidate with the highest $\epsilon$-best-belief on its anchor, optionally combined with an adversarial pool, and the winner stays frozen for the whole next epoch~(\cref{sec:method:utility}).

\subsubsection{Evaluator replacement rule}
\label{app:swap-rule}

At checkpoint $c$, slot $m$ compares its incumbent $e_m^{(j_m)}$ with the challengers $\mathcal C_m(c)$ that have anchor evidence on $\mathrm{GT}_m$, so the candidates are $\mathcal E_m(c)=\{e_m^{(j_m)}\}\cup\mathcal C_m(c)$. With $S^{\mathrm{gt}}_e$ and $F^{\mathrm{gt}}_e$ the successes and failures of $e$ on the anchor, each candidate scores
\[
BB_\epsilon(e)=I^{-1}_{\epsilon}\bigl(1+S^{\mathrm{gt}}_e,\,1+F^{\mathrm{gt}}_e\bigr),
\]
the $\epsilon$-quantile of its Beta working posterior~(\cref{prop:best-belief-lower-bound}), the same form of score that selects agents~(\cref{sec:method:formal}). The slot promotes the candidate with the highest $BB_\epsilon$. Ties go to the incumbent, in which case no transition happens. A slot with an adversarial pool instead scores a weighted average of its best-beliefs on positive anchor items, negative anchor items, and the current pool~(\cref{app:campaign-theory}). The pool is rebuilt at each replacement from the displaced evaluator's accepted outputs~(\cref{app:campaign-theory}).

\subsection{Working-Posterior Consistency}
\label{app:proofs:role_task_balance}

Fix an epoch $\boldsymbol j$ and a node $a$. Let $\mathcal C$ be the set of eligible role--task cells $c=(r,d)$ of $a$, fixed during the epoch, and let $p_c$ be the success probability of one evaluation of $a$ on $c$. For weights $w_c\ge0$ summing to one, the epoch-local utility is $U_{\boldsymbol j}(a)=\sum_{c\in\mathcal C}w_cp_c$. The least-evaluated rule of \cref{sec:method:hierarchy} drives role counts, and task counts within a role, to equal shares, so it realizes the weights $w_{r,d}=1/(|\mathcal R_{\mathrm{elig}}|\,|\mathcal D_{r,\mathrm{elig}}|)$ of the utility in \cref{sec:method:formal}.

\begin{proposition}[Consistency under balanced role--task sampling]
\label{prop:balanced-working-posterior}
After $n$ evaluations of $a$ in the epoch, let $n_c(n)$ count those on cell $c$ and $S_a(n)$ count the successes. Assume $n_c(n)/n\to w_c$ for every $c$, and that each cell's running success rate converges almost surely to $p_c$, as it does when the cell's outcomes are independent Bernoulli$(p_c)$ draws. Then the mean $M_n=(1+S_a(n))/(n+2)$ of the working posterior $\operatorname{Beta}(1+S_a(n),\,1+n-S_a(n))$ converges almost surely to the epoch-local utility $U_{\boldsymbol j}(a)$ of the node.
\end{proposition}

\begin{proof}
Let $\widehat p_c$ be the success rate of cell $c$ over its $n_c(n)$ trials, so $S_a(n)/n=\sum_c (n_c(n)/n)\,\widehat p_c$. A cell with $w_c>0$ has $n_c(n)\to\infty$ and hence $\widehat p_c\to p_c$. A cell with $w_c=0$ contributes at most $n_c(n)/n\to0$. Since $\mathcal C$ is finite, $S_a(n)/n\to U_{\boldsymbol j}(a)$, and $|M_n-S_a(n)/n|=|n-2S_a(n)|/(n(n+2))\le1/(n+2)$.
\end{proof}

The pooled Beta thus preserves \hgm's interface, and its mean has the correct limit.

\subsection{Epoch-Local Fixed-Utility Validity}
\label{app:proofs:epoch-local-validity}

Each evaluator slot $m\in\{1,\dots,M\}$ has an active criterion $\kappa_m^{(j_m)}$. It collects every slot-local object that can change the distribution of an outcome that depends on the slot: the frozen evaluator, and any replay, validation, or adversarial distribution tied to it. Derived search statistics are not part of it. The epoch vector $\boldsymbol j=(j_1,\dots,j_M)$ indexes the active criteria $\boldsymbol\kappa^{\boldsymbol j}$.

\begin{assumption}[Stationary learned evaluation within an epoch]
\label{assumption:generative}
Within epoch $\boldsymbol j$, for every evaluator-dependent node--role--task cell $(a,r,d)$, the following are fixed: the active criteria that affect the cell, the workspace of $a$, the protocol that generates, replays, or samples the scored object for task $d$, and the binary scoring rule. Equivalently, the object and outcome are drawn as
\[
X\sim P^{\boldsymbol j}_{a,r,d},\qquad O\sim K^{\boldsymbol j}_{r,d}(\cdot\mid X),\qquad O\in\{0,1\},
\]
where both laws are time-homogeneous and unaffected by earlier outcomes in the epoch. A deterministic evaluator has $K^{\boldsymbol j}_{r,d}(1\mid x)\in\{0,1\}$.
\end{assumption}

This is the evaluator-dependent form of \cref{cond:stationary}. It fixes each outcome's mean, all that \cref{prop:validity} uses.

A \emph{utility record} $z=(a,r,d,o,\operatorname{dep}(z),\kappa(z))$ stores the node, role, task, and binary outcome. It also stores $\operatorname{dep}(z)\subseteq\{1,\dots,M\}$, the slots whose criteria affected the record through generation, sampling, replay, or scoring, and $\kappa(z)$, the criteria active on those slots when it was written. Evaluator-independent records have $\operatorname{dep}(z)=\emptyset$. Raw artifacts and audit logs are not records. They may be kept across transitions, but reach search statistics only through a new record.

\begin{definition}[Criterion-valid record]
\label{def:criterion-valid-record}
A record $z$ is \emph{valid} under $\boldsymbol\kappa^{\boldsymbol j}$ if $\kappa_\ell(z)=\kappa_\ell^{(j_\ell)}$ for every $\ell\in\operatorname{dep}(z)$. Records with $\operatorname{dep}(z)=\emptyset$ are always valid. An archive is \emph{criterion-consistent} under $\boldsymbol\kappa^{\boldsymbol j}$ if every record it retains is valid under those criteria.
\end{definition}

\begin{definition}[Utility transition on slot \(m\)]
\label{def:utility-transition}
A utility transition on slot $m$ moves to the epoch vector $\boldsymbol j'$ that differs from $\boldsymbol j$ only in $j'_m=j_m+1$. The new criterion $\kappa_m^{(j'_m)}$ is frozen before any new record is collected. The transition then filters the records $\mathcal Z_a$ of every node to
\[
\operatorname{Erase}_m(\mathcal Z_a)=\bigl\{z\in\mathcal Z_a:\ m\notin\operatorname{dep}(z)\ \text{or}\ \kappa_m(z)=\kappa_m^{(j'_m)}\bigr\}
\]
and recomputes every derived statistic from the retained records: node and clade success and failure counts, role and task counters, Thompson-sampling statistics, slot-local validation counters, and any cached score summaries derived from them.
\end{definition}

\begin{proposition}[Selective erasure preserves criterion consistency]
\label{prop:erasure-stationarity}
If the archive is criterion-consistent under $\boldsymbol\kappa^{\boldsymbol j}$ before a transition on slot $m$, then after it the archive is criterion-consistent under $\boldsymbol\kappa^{\boldsymbol j'}$, and it stays so as the new epoch adds records tagged with the active criteria.
\end{proposition}

\begin{proof}
Take a retained record $z$ and $\ell\in\operatorname{dep}(z)$. If $\ell=m$, retention means $\kappa_m(z)=\kappa_m^{(j'_m)}$. If $\ell\ne m$, prior consistency gives $\kappa_\ell(z)=\kappa_\ell^{(j_\ell)}=\kappa_\ell^{(j'_\ell)}$. New records are valid by construction, and statistics are recomputed from valid records only.
\end{proof}

\begin{proposition}[Epoch-local fixed-criterion validity]
\label{prop:validity}
Fix $\boldsymbol j$ and an archive that is criterion-consistent under $\boldsymbol\kappa^{\boldsymbol j}$. Let evaluator-dependent roles satisfy \cref{assumption:generative}; evaluator-independent roles satisfy \cref{cond:stationary}. Then every eligible $(a,r,d)$ has a success probability $p_{r,d,\boldsymbol j}(a)=\Pr_{\boldsymbol\kappa^{\boldsymbol j}}(O=1\mid a,r,d)$ that is fixed throughout the epoch. For fixed weights $w^{\boldsymbol j}_{r,d}\ge0$ normalized over the eligible cells $\mathcal C_{\boldsymbol j}(a)$, the utility
\[
U_{\boldsymbol j}(a)=\sum_{(r,d)\in\mathcal C_{\boldsymbol j}(a)}w^{\boldsymbol j}_{r,d}\,p_{r,d,\boldsymbol j}(a)
\]
is fixed as well. The epoch is thus a fixed binary-outcome search problem. \Cref{thm:cmp} therefore applies to this problem under $U_{\boldsymbol j}$, and the Thompson sampler estimates it from retained records.
\end{proposition}

\begin{proof}
For an evaluator-dependent role, $p_{r,d,\boldsymbol j}(a)=\mathbb E_{X\sim P^{\boldsymbol j}_{a,r,d}}[K^{\boldsymbol j}_{r,d}(1\mid X)]$, and both laws are fixed in the epoch. \Cref{cond:stationary} gives the same for evaluator-independent roles. By criterion consistency, every retained record was written under the active criteria of the slots it depends on, and new records use the same frozen criteria. Old and new records are therefore outcomes of the same fixed problem.
\end{proof}

\subsection{Supporting Results for Controlled Utility Evolution}
\label{app:controlled_utility_proofs}

\begin{remark}[Anchor-only evaluator replacement]
\label{prop:anchor-best-belief-replacement}
On a slot without a pool, the score $BB_\epsilon(e)$ of \cref{app:swap-rule} depends on the anchor counts $(S^{\mathrm{gt}}_e,F^{\mathrm{gt}}_e)$ alone; no slot-dependent record enters it. The incumbent is a candidate and wins ties, so a transition happens only when a challenger scores strictly higher, and the promoted evaluator's anchor best-belief is at least the incumbent's.
\end{remark}

\begin{proposition}[Best-belief lower-bound trajectory]
\label{prop:best-belief-lower-bound}
Let $\mathcal A$ be a finite candidate set, each $a\in\mathcal A$ with working posterior $\operatorname{Beta}(1+S_a,1+F_a)$ over its epoch-local utility and $BB_\epsilon(a)=I^{-1}_\epsilon(1+S_a,1+F_a)$. For $a^\star\in\argmax_{a\in\mathcal A}BB_\epsilon(a)$, the working posterior of $a^\star$ gives $\Pr\bigl(U(a^\star)\ge BB_\epsilon(a^\star)\bigr)=1-\epsilon$. For $L$ reported selections, if each selected agent's utility falls below its bound with probability at most $\epsilon$ under the joint data and selection process, all $L$ bounds hold simultaneously with probability at least $1-L\epsilon$ by a union bound.
\end{proposition}

\begin{proof}
The first claim is the definition of the $\epsilon$-quantile of a continuous distribution. The second is the union bound over the $L$ violation events.
\end{proof}

\Cref{app:campaign-theory} gives the conditions under which best-belief approximates exact comparison.

\noindent\textbf{Piecewise fixed-criterion validity.}~~Let $\tau_1<\tau_2<\cdots\le B$ be the realized transition times, with $\tau_0=0$. Let $\mathcal F_{\tau_k}$ be the history at $\tau_k$: the archive, records, logs, candidates, checkpoint statistics, and all earlier random choices made by the search process.

\begin{proposition}[Piecewise fixed-criterion validity]
\label{prop:piecewise-fixed-validity}
Suppose \cref{assumption:generative} holds in every epoch, transitions happen only at checkpoints, and each transition erases and recomputes as in \cref{def:utility-transition}. Then, conditional on $\mathcal F_{\tau_k}$ and the criteria selected at $\tau_k$, the interval $[\tau_k,\tau_{k+1})$ is a fixed-criterion binary-outcome search problem. Each eligible $(a,r,d)$ has a fixed success probability $p^{(k)}_{a,r,d}=\Pr(O=1\mid a,r,d,\kappa^{(k)})$, and the role--task balanced utility is fixed for the epoch.
\end{proposition}

\begin{proof}
Induct on $k$. At $\tau_0$ every record was written under the initial criteria, so the archive is criterion-consistent. If it is consistent before $\tau_k$, apply the transitions at $\tau_k$ one slot at a time; \cref{prop:erasure-stationarity} keeps the archive consistent after each and through the epoch. No transition happens inside $[\tau_k,\tau_{k+1})$, so \cref{prop:validity} gives the fixed probabilities. The tree inherited from earlier epochs is part of $\mathcal F_{\tau_k}$. It fixes which nodes exist but changes no outcome law inside the interval.
\end{proof}

The task-agent objective changes at every transition, so this statement is per epoch.

\noindent\textbf{Exponential checkpoint overhead.}~~Fix $\rho>1$ and $h\ge1$, and place checkpoint opportunities at
\[
\mathcal C_B(\rho,h)=\{\lfloor h\rho^q\rfloor: q=0,1,\dots,Q\},
\]
where $Q$ is the largest integer with $h\rho^Q\le B$.

\begin{proposition}[Linear work under exponential checkpoints]
\label{prop:linear-work}
Let each evaluation write at most one slot-dependent record, so a checkpoint at time $c$ reprocesses at most $c$ records. Then all checkpoints in $\mathcal C_B(\rho,h)$ together reprocess $\mathcal O(B)$ records, at most $\rho B/(\rho-1)$, counting a record once for each checkpoint that reprocesses it.
\end{proposition}

\begin{proof}
$\sum_{q=0}^{Q}h\rho^q=h(\rho^{Q+1}-1)/(\rho-1)<\rho\,h\rho^Q/(\rho-1)\le\rho B/(\rho-1)$. The base-two schedule, $\rho=2$, gives $2B$.
\end{proof}

\begin{remark}[Dense uniform checkpointing]
\label{rem:dense-quadratic}
With a checkpoint every $h$ steps, the $q$-th can reprocess $qh$ records, and the $\lfloor B/h\rfloor$ checkpoints total $h\sum_{q=1}^{\lfloor B/h\rfloor}q=\Theta(B^2/h)$ for $B\ge h$.
\end{remark}

\subsection{Evaluator Improvement}
\label{app:campaign-theory}

\subsubsection{Exact validation}

Fix one slot, a finite task set $\mathcal D$, finite validation sets $D_d$ with fixed labels, and sampling weights $w_d\ge0$ summing to one. Each immutable evaluator $e$ has a fixed correctness table $y_d(e,x)\in\{0,1\}$, and
\begin{equation}
\label{eq:campaign-fixed-objective}
Q(e)=\sum_{d\in\mathcal D}\frac{w_d}{|D_d|}\sum_{x\in D_d}y_d(e,x).
\end{equation}

\begin{proposition}[Exact replacement]
\label{prop:campaign-exact}
At checkpoint $j$, let the finite comparison set $\mathcal E_j$ contain the incumbent $e^{(j)}$, and choose $e^{(j+1)}\in\argmax_{e\in\mathcal E_j}Q(e)$, keeping the incumbent on ties. Then $Q(e^{(j+1)})\ge Q(e^{(j)})$, with strict inequality at every replacement, and there are finitely many.
\end{proposition}
\begin{proof}
The incumbent is a candidate, so $Q(e^{(j+1)})=\max_{\mathcal E_j}Q\ge Q(e^{(j)})$, and a tie keeps it. Each replacement strictly raises $Q$, and $Q$ takes finitely many values on a finite binary table, so replacements stop. With $N$ equally weighted items, each increase is at least $1/N$, so there are at most $N$.
\end{proof}

The guarantee concerns the validation objective and survives any shortlist containing the incumbent.

\subsubsection{Best-belief approximation}

\begin{proposition}[Comparison error]
\label{prop:campaign-approximation}
If $|b_j(e)-Q(e)|\le\eta_j$ on $\mathcal E_j$ and $e^{(j+1)}$ maximizes $b_j$, then
\begin{equation}
\label{eq:campaign-comparison-error}
Q(e^{(j+1)})\ge\max_{e\in\mathcal E_j}Q(e)-2\eta_j\ge Q(e^{(j)})-2\eta_j.
\end{equation}
\end{proposition}
\begin{proof}
Let $e^\star$ maximize $Q$ over $\mathcal E_j$. Then $Q(e^{(j+1)})\ge b_j(e^{(j+1)})-\eta_j\ge b_j(e^\star)-\eta_j\ge Q(e^\star)-2\eta_j$.
\end{proof}

Best-belief meets this condition when it targets the same objective. For one candidate with $n=S_n+F_n$ trials, $Z_n\sim\operatorname{Beta}(1+S_n,1+F_n)$ has mean $M_n=(1+S_n)/(n+2)$ and $\epsilon$-quantile $b_n$ with
\begin{equation}
\label{eq:campaign-quantile-concentration}
|b_n-M_n|\le\frac{1}{\sqrt{4(n+3)\min\{\epsilon,1-\epsilon\}}},
\end{equation}
because $\operatorname{Var}(Z_n)=M_n(1-M_n)/(n+3)\le1/(4(n+3))$ and Chebyshev gives $\epsilon=\Pr(Z_n\le b_n)\le\operatorname{Var}(Z_n)/(M_n-b_n)^2$ when $b_n<M_n$ (use $1-\epsilon$ when $b_n>M_n$). So if $S_n/n\to Q(e)$ almost surely, then $M_n\to Q(e)$ and $b_n\to Q(e)$ for fixed $\epsilon\in(0,1)$, and \cref{prop:campaign-approximation} applies with $\eta_j\to0$.

\subsubsection{Adversarial objectives}

A slot with an adversarial pool scores the objective $Q_P(e)=p_+(e)/2+p_-(e)/4+r_P(e)/4$, where $p_+$ and $p_-$ are accuracies on positive and ordinary-negative anchor items and $r_P$ is the rejection rate on pool $P$. Within an epoch the pool is fixed, so \cref{prop:campaign-exact,prop:campaign-approximation} apply to $Q_P$ at each checkpoint. Across checkpoints the pool changes, so improvement is measured against the current objective: with drift $\delta_j=\sup_{e}|Q_{j+1}(e)-Q_j(e)|$ and $b_j$ within $\eta_j$ of $Q_j$,
\begin{equation}
\label{eq:campaign-drift}
Q_{j+1}(e^{(j+1)})\ge Q_j(e^{(j)})-2\eta_j-\delta_j
\end{equation}
by \cref{eq:campaign-comparison-error}. The anchor terms $p_+$ and $p_-$ never change, so the adversarial term only reweights it.

\section{Limitations}
\label{app:limitations}

Our analysis assumes stationarity within each epoch~(\cref{app:proofs:epoch-local-validity}); provider-side model updates or nondeterministic tools can weaken this assumption in practice.

The reviewer panel measures cross-reviewer acceptance, and the generated papers and proofs were scored by automated panels rather than human experts.

The anchor bounds what evaluator evolution can achieve. A weak, noisy, or biased anchor lets evaluators drift across epochs, and \apres decisions and \imogradbench grades are themselves imperfect. A good anchor keeps evaluators accurate but confines them to its own decision boundary, so gains beyond the benchmark must come from objectives layered on top of it.

\section{Extended Related Work}
\label{app:related}

This section extends \Cref{sec:related} with a broader overview.

\paragraph{Self-Improving AI.}
\citet{DBLP:journals/ac/Good65} proposed a machine that designs better machines long before one was feasible, and \citet{schmidhuber2003godel} made the idea precise with the G\"odel Machine. Of its empirical successors~(\cref{sec:related}), \hgm~\citep{wang2025huxley} starts from the observation that an agent's benchmark score poorly predicts the value of its descendants, and reports human-level coding-agent performance while scoring clades instead. \hyperagents~\citep{zhang2026hyperagents} pairs each node's task agent with a meta-agent that improves it, so the improvement mechanism is itself editable. Related systems self-modify scaffolds or runtime logic~\citep{xia2025live,zelikman2024self,yin2025godel}, refine prompts~\citep{fernando2023promptbreeder} or reasoning traces~\citep{zelikman2022star}, self-tune for a domain~\citep{shen2026star}, or hill-climb~\citep{robeyns2025self}. Two surveys organize the field, as a closed lifecycle of data, selection, optimization, and inference supervised by an evaluation layer~\citep{yang2026selfimprovement}, and by \emph{what} is improved and \emph{how fully} the loop is closed~\citep{chen2026rsi}. Both describe fixed evaluation criteria as the norm.

\paragraph{Self-Improvement in Weight Space and Self-Play.}
A parallel line closes the loop through parameter updates instead of code edits. Meta-learning optimizes the learning dynamics themselves~\citep{kirsch2022eliminating,lu2023arbitrary}, while inference-time refinement improves single outputs through verbal self-critique without keeping gains across tasks~\citep{DBLP:conf/nips/ShinnCGNY23}. Rejection-sampling self-training~\citep{zelikman2022star,DBLP:conf/emnlp/0001GHW00023}, optionally guided by learned self-critique~\citep{DBLP:journals/corr/abs-2310-00533}, and self-play fine-tuning~\citep{chen2024spin} bootstrap a policy from its own filtered generations. Self-rewarding models~\citep{yuan2024selfreward} also act as their own judge; SCoRe~\citep{DBLP:conf/iclr/KumarZASCSBIBRZ25} trains self-correction directly, related work rewards the reflection step~\citep{DBLP:journals/corr/abs-2505-24726}, and SEAL~\citep{DBLP:journals/corr/abs-2506-10943} learns to write its own fine-tuning data and update directives. Proposer--solver(--judge) frameworks generate their own curricula: R-Zero~\citep{huang2026rzero}, Absolute Zero~\citep{DBLP:conf/nips/ZhaoWWXYLWWZH25}, SPICE~\citep{liu2025spice}, and Language Self-Play~\citep{kuba2025lsp} pose increasingly hard challenges, Multi-Agent Evolve~\citep{chen2025multi} trains proposer, solver, and judge from one base model, and CURE~\citep{wang2025cure} co-evolves a coder with its unit tester. Closest to our work, SCORE~\citep{zhu2026selfevolvingdeepresearchjoint} observes that static evaluators saturate and co-trains evaluator and solver with shared parameters, and EvoRubrics~\citep{ding2026evorubricsdynamicrubricsrewards} co-evolves a rubric generator adversarially with the policy. These methods mostly co-adapt generator and reward \emph{inside gradient-based training}. \Theourmethod instead searches over agentic codebases and gates evaluator promotion on an external anchor, in line with evidence that self-correction is unreliable without external feedback~\citep{DBLP:conf/iclr/GouSGSYDC24}.

\paragraph{Learned Evaluation and Its Biases.}
Our evaluators build on LLM- and agent-as-a-judge evaluation~\citep{zheng2023judging,zhuge2024agent,you2026agentjudgesurvey}, benchmarked for agreement with human preference~\citep{zhao2026apres}. Judges show position, verbosity, and self-preference biases~\citep{zheng2023judging,dubois2024lengthcontrolled,wataoka2024selfpreference,pombal2026rubric}; self-preference tracks a model's ability to recognize its own text~\citep{DBLP:conf/nips/PanicksseryBF24} and persists even when the judge's own output is worse~\citep{chen2025hspp}. Peer review, our paper domain, is especially exposed: LLM reviewers lean toward acceptance~\citep{keuper2025injection}, favor prestigious authors~\citep{pataranutaporn2025aisolvepeerreview}, and are fooled by laundered papers~\citep{baumann2026stopautomatingpeerreview} and by agents optimized to write convincing but unsound work~\citep{jiang2025badscientist}. Mitigations such as juries~\citep{verga2024juries}, adversarial hardening~\citep{DBLP:conf/emnlp/RainaLG24}, and evaluator lifecycle management~\citep{zhang2026gradesgradercoevolvingevaluation} mostly keep the judge outside the system it scores. \citet{zhang2026gradesgradercoevolvingevaluation} come closest: they co-evolve metrics with agent skills under ground-truth guards and show the anchor keeps the metric reliable. \Theourmethod shares this anchor discipline but places it inside recursive self-improvement, and uses epoch boundaries to add adversarial objectives, such as replaying papers a displaced reviewer accepted, that target generated artifacts absent from the anchor.

\paragraph{Reward Hacking and Non-Stationary Objectives.}
Optimizing a fixed proxy invites exploitation. Reward hacking has been formalized~\citep{skalse2022defining} and follows Goodhart-style overoptimization curves against learned reward models~\citep{gao2023scaling}; a recent survey frames it as a consequence of compressing intent into a scalar~\citep{wang2026rewardhackingeralarge}, and policies learn proxy gaps as a reusable skill that adapts when the evaluator changes~\citep{beigi2026prime}. Contemporaneous theory argues that a static proxy has a receding ``verification horizon''~\citep{wang2026horizon}. Live benchmarks~\citep{jain2025livecodebench,kasai2023realtime} refresh test data against contamination and saturation~\citep{DBLP:journals/corr/abs-2203-04592,yang2026selfimprovement}, but benchmarks are costly to build, and changing one mid-run invalidates a search's evidence. Controlled utility evolution addresses this: the utility changes only at epoch boundaries, and selective erasure keeps each epoch a fixed-criterion search.

\paragraph{Multi-Agent Co-Evolution and Open-Endedness.}
Co-evolution against a moving target is old in evolutionary computation: fitness sharing~\citep{rosin1997new} and self-modifying policies~\citep{schmidhuber1996multi} co-adapt populations, and self-play trains an agent against stronger versions of itself~\citep{silver2016mastering,silver2018general}. LLM-agent versions share experience across populations~\citep{weng2026group}, play adversarial strategy--environment games~\citep{DBLP:conf/gecco/SygkounasHPML26}, co-evolve evaluators from opponent pools~\citep{lv2026famou}, or jointly optimize interacting components~\citep{liu2026escher}, and human--AI co-improvement adapts the signal through human feedback~\citep{weston2025coimprovement}. Open-ended search aims at endless novelty~\citep{stanley2017openendedness,hughes2024position}, often through quality-diversity archives~\citep{DBLP:journals/corr/MouretC15,DBLP:journals/nature/EcoffetHLSC21} or growing skill libraries~\citep{DBLP:journals/tmlr/WangX0MXZFA24}, increasingly with foundation models~\citep{lu2024ai,novikov2025alphaevolve,romera2024mathematical,huautomated} and toward a closed self-referential loop~\citep{clune2019ai}. \Theourmethod instead changes the learned evaluation criterion itself at explicit epoch boundaries, with epoch-local validity.

\paragraph{Automated Scientific Discovery.}
Our open-ended domains belong to the effort to automate research end to end. The AI Scientist produced the first fully AI-generated paper to pass workshop peer review~\citep{lu2024ai,yamada2025ai}, and Co-Scientist~\citep{gottweis2025coscientist}, Robin~\citep{ghareeb2026robin}, AlphaEvolve~\citep{novikov2025alphaevolve}, and evolutionary multi-agent scientists~\citep{lyu2026evoscientist,schmidgall2025agentrxiv} now cover hypothesis generation through experiments and write-up. These pipelines typically fix their model judges per run, and critiques highlight automated evaluation as a key reliability risk~\citep{baumann2026stopautomatingpeerreview,luo2025pitfalls}. \Theourmethod lets the reviewer, grader, and code reviewer co-evolve with the writers, provers, and coders they score, anchored to \apres decisions~\citep{zhao2026apres}, \imogradbench grades~\citep{luong2025towards}, and accept/reject labels from \crave~\citep{zhang2025crave}.

\FloatBarrier

\end{document}